%% file: main.tex
\documentclass[11pt]{article}

\pdfoutput=1

\input{header}
\usepackage{tgpagella}
\usepackage{float}
\usepackage{afterpage}
\usepackage{accents}
\usepackage{bm}
\usepackage{centernot}
\usepackage[nottoc]{tocbibind}
\usepackage{tocloft}
\restylefloat{table}
\usepackage{etoolbox}
\usepackage{nicematrix}
\usepackage{caption}
\usepackage{xspace}
\usepackage{bbm}
\usepackage{cleveref}
\usepackage{pgfplots}
\usepackage{soul}

\usepackage{wrapfig}
\usepackage{subcaption}

\usepackage{booktabs,tabularx}
\usepackage{changepage}

\usepackage{ragged2e}        
\usepackage{nicematrix}     

\newcommand{\TV}{\mathsf{TV}}
\newcommand{\ACCEPT}{\mathsf{Accept}}
\newcommand{\REJECT}{\mathsf{Reject}}

\newcommand{\SOPP}{\mathcal{SOPP}}

\renewcommand{\epsilon}{\varepsilon}

\newcommand{\unif}{\mathsf{unif}}
\newcommand{\Ber}{\mathsf{Ber}}
\DeclareMathOperator*{\argmax}{arg\,max}

\newtoggle{anonymous}
\togglefalse{anonymous}

\title{Feature Selection and Junta Testing are Statistically Equivalent}

\iftoggle{anonymous}{%
    }{%
\author{ \hspace*{-1.5em}
Lorenzo Beretta\\
\hspace*{-1em} \emph{UCSC}
\and
 Nathaniel Harms\\
\emph{EPFL}
\and
Caleb Koch\\
\emph{Stanford}
}
}

\begin{document}

\maketitle

\begin{abstract}
For a function $f \colon \{0,1\}^n \to \zo$, the junta testing problem asks whether $f$ depends on
only $k$ variables. If $f$ depends on only $k$ variables, the feature selection problem asks to find
those variables. We prove that these two tasks are statistically equivalent. Specifically, we show
that the ``brute-force'' algorithm, which checks for any set of $k$ variables consistent with
the sample, is simultaneously sample-optimal for both problems, and the optimal sample size is
\[
\Theta\left(\frac 1 \varepsilon \left( \sqrt{2^k \log {n \choose k}} + \log {n \choose
k}\right)\right).
\]
\end{abstract}

{
\setcounter{tocdepth}{2} 
\tableofcontents
}
\thispagestyle{empty}
\setcounter{page}{0}
\newpage
\setcounter{page}{1}

\input{intro}

\input{upper}

\input{lower}

\appendix

\input{appendix_lower}

\newpage

\iftoggle{anonymous}{%
}{%
\subsection*{Acknowledgments}
Much of \cref{section:intro-big-picture} came from discussions with Renato Ferreira Pinto Jr.
Thanks to Shivam Nadimpalli and Shyamal Patel for helpful discussions. Thanks to anonymous reviewers
for many helpful comments and pointers to the literature. Some of this work was done while Nathaniel
Harms and Caleb Koch were visiting the Simons Institute for the Theory of Computing.  Nathaniel
Harms was partially funded by a Simons Fellowship, and by the Swiss State Secretariat for Education,
Research, and Innovation (SERI) under contract number MB22.00026. Caleb Koch was partially funded by
NSF awards 1942123, 2211237, and 2224246.
}

\bibliographystyle{alphaurl}
\bibliography{references.bib}

\end{document}

%% file: header.tex
\usepackage[a4paper, portrait, margin=2cm]{geometry}
\usepackage[colorlinks=true,linkcolor=blue,citecolor=ForestGreen]{hyperref}
\usepackage{algorithm}
\usepackage[noend]{algpseudocode}
\usepackage{url}
\usepackage{amsmath,amssymb,amsthm}
\usepackage{thmtools,thm-restate}
\usepackage[noabbrev,capitalise,nameinlink]{cleveref}
\usepackage{mathtools}
\usepackage{xspace}
\usepackage{verbatim}
\usepackage{mathrsfs}
\usepackage[usenames,dvipsnames,svgnames,table]{xcolor}
\usepackage{pgf}
\usepackage[dvipsnames]{xcolor}
\usepackage{todo}
\usepackage{tabularx}
\usepackage{enumitem}
\usepackage{derivative}
\usepackage{bm}
\usepackage{multirow}
\usepackage{diagbox}
\usepackage{nicematrix}
\usepackage[most]{tcolorbox}
\usepackage{esint}

\usepackage{dsfont}

\newcommand*\ie{i.\kern.1em e.\ }
\newcommand*\eg{e.\kern.1em g.\ }
\newcommand*\cf{c.\kern.1em f.\ }
\newcommand*\almev{a.\kern.1em e.\ }

\definecolor{ama-iro}{RGB}{0, 158, 243.0}
\definecolor{fuyu-gaki}{RGB}{251, 74, 52}
\definecolor{momiji}{RGB}{245, 70, 111}
\definecolor{hotaru-bi}{RGB}{229,221,58} 
\definecolor{kon-peki}{RGB}{1,120,217}
\definecolor{shin-kai}{RGB}{77,98,152}
\definecolor{shin-ryoku}{RGB}{1,145,97}
\definecolor{yama-budo}{RGB}{171,14,122}

\definecolor{citecolor}{HTML}{5E9100}

\definecolor{theoremcolor}{HTML}{FFE1D9}
\definecolor{resultcolor}{HTML}{FFE1D9}

\definecolor{constraintcolor}{HTML}{BBD49B}
\definecolor{goalcolor}{HTML}{CAE4A7}

\definecolor{remarkcolor}{HTML}{E3EEC7}

\definecolor{definitioncolor}{HTML}{FCDFBE}

\definecolor{examplecolor}{HTML}{F9E9D9}
\definecolor{questioncolor}{HTML}{DDE8EB}

\definecolor{captioncolor}{RGB}{128,128,128}

\hypersetup{
  colorlinks=true,
  linkcolor=momiji,
  filecolor=ama-iro,
  urlcolor=momiji,
  citecolor=citecolor
}

\theoremstyle{plain}
\newtheorem{theorem}{Theorem}[section]
\newtheorem{lemma}[theorem]{Lemma}

\newtheorem{proposition}[theorem]{Proposition}
\newtheorem{claim}[theorem]{Claim}
\newtheorem{corollary}[theorem]{Corollary}
\newtheorem{question}[theorem]{Question}
\newtheorem{question*}{Question}

\newtheorem{observation}[theorem]{Observation}
\newtheorem{assumption}[theorem]{Assumption}

\crefname{claim}{Claim}{Claims}
\crefname{fact}{Fact}{Facts}

\theoremstyle{definition}

\newtheorem{definition}[theorem]{Definition}
\newtheorem{remark}[theorem]{Remark}
\newtheorem{example}[theorem]{Example}
\newtheorem{goal}[theorem]{Goal}
\newtheorem{goal*}{Goal}

\theoremstyle{plain}

\newenvironment{boxtheorem}{\begin{theorem}}{\end{theorem}}
\tcolorboxenvironment{boxtheorem}{colback=theoremcolor, colframe=white,
    colbacktitle=theoremcolor, coltitle=theoremcolor}

\newenvironment{boxlemma}{\begin{lemma}}{\end{lemma}}
\tcolorboxenvironment{boxlemma}{colback=resultcolor, colframe=white,
    colbacktitle=resultcolor, coltitle=resultcolor}

\newenvironment{boxproposition}{\begin{proposition}}{\end{proposition}}
\tcolorboxenvironment{boxproposition}{colback=resultcolor, colframe=white,
    colbacktitle=resultcolor, coltitle=resultcolor}

\tcolorboxenvironment{boxcorollary}{colback=resultcolor, colframe=white,
    colbacktitle=resultcolor, coltitle=resultcolor}

\tcolorboxenvironment{boxobservation}{colback=resultcolor, colframe=white,
    colbacktitle=resultcolor, coltitle=resultcolor}

\newenvironment{boxquestion}{\begin{question}}{\end{question}}
\tcolorboxenvironment{boxquestion}{colback=questioncolor, colframe=white,
    colbacktitle=questioncolor, coltitle=shin-kai}
\newenvironment{boxquestion*}{\begin{question*}}{\end{question*}}
\tcolorboxenvironment{boxquestion*}{colback=questioncolor, colframe=white,
    colbacktitle=questioncolor, coltitle=questioncolor}

\newenvironment{boxdefinition}{\begin{definition}}{\end{definition}}
\tcolorboxenvironment{boxdefinition}{colback=definitioncolor, colframe=white,
    colbacktitle=definitioncolor, coltitle=definitioncolor}

\tcolorboxenvironment{boxassumption}{colback=definitioncolor, colframe=white,
    colbacktitle=definitioncolor, coltitle=definitioncolor}

\tcolorboxenvironment{boxexample}{colback=examplecolor, colframe=white,
    colbacktitle=examplecolor, coltitle=examplecolor}

\newtheorem{exercise}[theorem]{Exercise}

\tcolorboxenvironment{boxexercise}{colback=exercisecolor, colframe=white,
    colbacktitle=exercisecolor, coltitle=exercisecolor}

\tcolorboxenvironment{boxgoal}{colback=goalcolor, colframe=white,
    colbacktitle=goalcolor, coltitle=goalcolor}

\newenvironment{boxgoal*}{\begin{goal*}}{\end{goal*}}
\tcolorboxenvironment{boxgoal*}{colback=goalcolor, colframe=white,
    colbacktitle=goalcolor, coltitle=goalcolor}

\tcolorboxenvironment{boxremark}{colback=remarkcolor, colframe=white,
    colbacktitle=remarkcolor, coltitle=remarkcolor}

\newcommand{\ignore}[1]{}

\newcommand{\dist}{\mathsf{dist}}

\newcommand{\VC}{\mathsf{VC}}

\newcommand{\Ex}[1]{\bE \left[ #1 \right]}
\newcommand{\Exu}[2]{\underset{#1} \bE \left[ #2 \right] }
\newcommand{\Exuc}[3]{\underset{#1} \bE \left[ #2 \;\; \left| \;\; #3
\right.\right] }

\renewcommand{\Pr}[1]{\bP \left[ #1 \right]} 
\newcommand{\Pru}[2]{\underset{ #1 }\bP \left[ #2 \right]}
\newcommand{\Pruc}[3]{\underset{ #1 }\bP \left[ #2 \;\; \mid \;\; #3 \right]}

\newcommand*{\define}{\mathrel{\vcenter{\baselineskip0.5ex \lineskiplimit0pt
                      \hbox{\scriptsize.}\hbox{\scriptsize.}}}%
                      =}

\newcommand{\inn}[1]{\langle #1 \rangle}

\newcommand{\ind}[1]{\mathds{1} \left[ #1 \right] }

\newcommand{\zo}{\{0,1\}}
\newcommand{\pmset}{\{\pm 1\}}

\newcommand{\cD}{\ensuremath{\mathcal{D}}}

\newcommand{\cF}{\ensuremath{\mathcal{F}}}

\newcommand{\cH}{\ensuremath{\mathcal{H}}}

\newcommand{\cJ}{\ensuremath{\mathcal{J}}}

\newcommand{\bE}{\ensuremath{\mathbb{E}}}

\newcommand{\bN}{\ensuremath{\mathbb{N}}}
\newcommand{\bP}{\ensuremath{\mathbb{P}}}
\newcommand{\bR}{\ensuremath{\mathbb{R}}}


%% file: intro.tex
\section{Introduction}

Humans throughout history, and computers more recently, have faced the task of determining which
information is relevant for some goal. Blum writes, ``nearly all results in machine learning,
whether experimental or theoretical, deal with problems of separating relevant from irrelevant
information'' \cite{Blu94}. 
For example, let's say we are given limited access to some function $f
\colon \zo^n \to \zo$; imagine that each $x \in \zo^n$ is the medical record of a patient and
$f(x)=1$ if they have disease $X$. We want to know about disease $X$, but not all information about
a patient may be relevant. So we ask:
\begin{enumerate}[itemsep=0pt]
    \item Does $f$ depend on only $k < n$ variables?
    \item If so, which $k$ variables?
\end{enumerate}
We don't have direct access to $f$, nor can we query $f(x)$ on an arbitrary input $x$, because we
can't just make up a patient and see if they have the disease. So, as in the standard PAC learning
model of Valiant \cite{Val84}, we assume that we see only random examples of the form $(\bm x, f(\bm x))$ where
$\bm x$ is drawn from an unknown probability distribution $\cD$ over $\zo^n$. This is the
\emph{distribution-free sample-based} model. In this model, given $m$ random examples $(\bm x, f(\bm
x))$, a parameter $k$, and a distance parameter $\epsilon > 0$, questions 1-2 may be formalized as:
\begin{enumerate}[itemsep=0pt]
    \item \textbf{Testing $k$-Juntas:} Output $\ACCEPT$ with probability $3/4$ if $f$ depends on
        only $k$ variables (\ie $f$ is a \emph{$k$-junta}), and output $\REJECT$ with probability
        $3/4$ if $f$ is \emph{$\epsilon$-far} from being a $k$-junta, meaning that for all
        $k$-juntas $g$,
        \[
            \epsilon < \dist_\cD(f,g) \define \Pru{\bm x \sim \cD}{f(\bm x) \neq g(\bm x)} .
        \]
    \item \textbf{$k$-Feature selection:} Assuming $f$ is a $k$-junta, output a set $T \subseteq
        [n]$ of $k$ variables such that (with probability $3/4$) there exists a $k$-junta $g$ on
        variables $T$ satisfying $\dist_\cD(f,g) < \epsilon$.
\end{enumerate}
These two fundamental problems do not easily reduce to each other\footnote{The standard
testing-by-learning reduction of \cite{GGR98} does not work to reduce testing juntas to feature
selection.}, are not obviously equivalent, and tight bounds on the required sample sizes are not
known.  However, they can both be solved by the same obvious algorithm, provided that the sample
size is large enough:

\begin{algorithm}
  \caption{Obvious algorithm}
  \label{alg:obvious}
  \begin{algorithmic}[1]
    \State Draw $m$ samples $S = \{ (x_i, f(x_i)) \mid i \in [m] \}$.
    \ForAll{sets of variables $T \in \binom{[n]}{k}$}
      \State Check if $S$ rules out $T$, \ie check for $x_i, x_j$ where $f(x_i) \neq f(x_j)$ but
        $x_i, x_j$ match on variables $T$, proving that $f$ does not
        depend only on variables $T$.
    \EndFor
    \If{any $T$ are not ruled out}
      \State Output any such $T$ (or output $\ACCEPT$ if you are testing)
    \Else
      \State Output $\REJECT$
    \EndIf
  \end{algorithmic}
\end{algorithm}
\ignore{
\begin{itemize}
\item Collect your set of $m$ samples $S = \{ (x_i, f(x_i)) \mid i \in [m] \}$.
\item For every 
set $T \subset [n]$ of $|T|=k$ variables, check if your samples rule out $T$, \ie check for
two points $x_i, x_j \in S$ which have the same values in $T$ and yet $f(x_i) \neq f(x_j)$.
\item If any $T$ remains, use it as the set of $k$ relevant variables for feature selection,
or output $\ACCEPT$ if you are testing. If no $T$ remains, output $\REJECT$.
\end{itemize} }
This algorithm is far from optimal for testing juntas when the algorithm is allowed to
make adaptive queries \cite{Bla09,Bsho19}.  But in the distribution-free sample-based setting, we show
that it is simultaneously sample-optimal for both feature selection and junta testing,
establishing that they are statistically equivalent:
\begin{boxtheorem}[Informal]
\label{thm:intro-informal}
Testing $k$-juntas and $k$-feature selection are statistically equivalent: they require the same
sample size
$m = \Theta\left(\frac{1}{\epsilon} \left( \sqrt{2^k \log\binom{n}{k}} +
\log\binom{n}{k}\right)\right)$,
and \cref{alg:obvious} is sample-optimal for both.
\end{boxtheorem}
To prove this theorem, we will prove tight lower bounds for both problems, and an upper bound on the
obvious algorithm. This is the first tight bound for testing any natural class of boolean functions
in the distribution-free sample-based model (see \cref{section:intro-big-picture}), and it improves
on the analyses of \cite{AHW16,BFH21}. For constant $\epsilon > 0$, our lower bound holds for the
uniform distribution.

\subsection{Counterfactual Worlds With Interesting Algorithms}

In the standard PAC learning model, if we want to \emph{learn} an unknown $k$-junta, it is known
that the ``obvious algorithm'' --- \ie output any $k$-junta $g$ that is not ruled out by the samples
--- is sample-optimal, requiring $\Theta\left(\tfrac{1}{\epsilon}(2^k + \log\binom{n}{k})\right)$
samples. In fact, for any hypothesis class $\cH$, there is an ``obvious algorithm'' that generalizes
\cref{alg:obvious}:

\begin{algorithm}
  \caption{Obvious algorithm for testing or learning $\cH$}
    \label{alg:erm}
  \begin{algorithmic}[1]
    \State Draw $m$ samples $S = \{ (x_i, f(x_i)) \mid i \in [m] \}$.
    \ForAll{functions $g \in \cH$}
      \State Check if $S$ rules out $g$, \ie check if $f(x_i) \neq g(x_i)$ for some $x_i \in S$.
    \EndFor
    \If{any $g \in \cH$ are not ruled out}
      \State Output any such $g$ (or output $\ACCEPT$ if you are testing)
    \Else
      \State Output $\REJECT$
    \EndIf
  \end{algorithmic}
\end{algorithm}

This algorithm is sample-optimal for learning \emph{any} class $\cH$, up to a $\log(1/\epsilon)$
factor\footnote{The $\log(1/\epsilon)$ factor gap depends both on whether the consistent output $g$
is chosen in a clever way \cite{AO07}, and whether we allow \emph{improper} learning, \ie outputting
a function which may not belong to $\cH$ \cite{Han16,Lar23}.} \cite{BHW89,EHKV89,AO07,Han16,Lar23}.
Let us clarify that \cref{alg:erm} may require different sample size depending on whether we want it
to solve the testing problem or the learning problem, although the algorithm itself remains the
same. By analogy, one may therefore wonder if it is also optimal for the associated \emph{testing}
problems.

Surprisingly, the answer is no, as seen in \cite{GR16,FH25} for testing \emph{support-size}: testing
if $f \colon [n] \to \zo$ takes value 1 on at most $k$ points, or if it is $\epsilon$-far from this
property (\ie $\dist_\cD(f,g) > \epsilon$ for all $g$ taking value 1 on at most $k$ points).  To
solve this problem, one may use only $O(\tfrac{k}{\epsilon \log k}\log(1/\epsilon))$ samples,
whereas \cref{alg:erm} requires $\Theta(k/\epsilon)$ samples. Building on breakthroughs of
\cite{VV11stoc,VV17,WY19} for estimating the support size of probability distributions, the improved
tester uses Chebyshev polynomials as estimators to avoid learning the function $f$, the histogram of
the underlying distribution $\cD$, or even an estimate of the support size\footnote{The algorithm
finds a good enough lower bound on the support size, with no corresponding upper bound, see
\cite{FH25}.}.

\cref{thm:intro-informal} rules out any similar tricks for testing juntas.To elaborate on this,
consider two points:
\begin{enumerate}
\item We obtain the upper bound for the obvious tester in the obvious way: determine the number of
samples required to rule out a single set $T$ with failure probability at most $\binom{n}{k}^{-1}$,
and then apply the union bound over all $\binom{n}{k}$ sets.
\item The best lower bound for testing juntas, prior to this work, is
\begin{equation}
\label{eq:intro-easy-bound}
    m = \Omega\left( \sqrt{2^k} + \log\binom{n}{k} \right) ,
\end{equation}
from \cite{AHW16}. We get $\sqrt{2^k}$ because this is the number of
samples required to find a pair $(x,y)$ that match on a fixed set of $k$ variables (a birthday
paradox argument), and we get $\log\binom{n}{k}$ because we need to rule out all $\binom{n}{k}$
parity functions on $k$ bits.
\end{enumerate}
In a counterfactual world where testing saves a only a $\sqrt{\log n}$ or even any $\omega(1)$
factor over feature selection, the tester \emph{must not} be learning the relevant variables, and
must rely on more interesting algorithmic techniques.  \cref{eq:intro-easy-bound} does not rule out
the possibility that a tester could, say, use correlations between sets of variables to avoid the
union bound analysis. To rule out these possibilities, prove optimality of the obvious algorithm,
and establish equivalence between junta testing and feature selection, we need bounds that are tight
up to constant factors. 

\subsection{The Big (Small) Picture: Fine-Grained Testing vs.~Learning}
\label{section:intro-big-picture}
Juntas are one of the most fundamental classes of functions in property testing (see
\cref{section:intro-comparison} for references to several prior works), so it is valuable to obtain
tight bounds, but \cref{thm:intro-informal} is also a piece of the larger puzzle of understanding
decision vs.~search problems with random samples. \cref{thm:intro-informal} advances a line of work
\cite{GGR98,GR16,BFH21,FH23,FH25} on the \emph{testing vs.~learning} question of Goldreich,
Goldwasser, and Ron \cite{GGR98}, in the distribution-free sample-based model corresponding to
standard PAC learning:

\begin{boxquestion}[Testing vs.~Learning]
\label{question:testing-vs-learning}
For which classes $\cH$ can testing be performed with fewer samples than learning?
\end{boxquestion}

This is an instance of the classic decision vs.~search dichotomy.  PAC learning is perhaps the most
well-understood model of learning, and the sample size required for learning any class $\cH$ is
between $\Theta(\tfrac{\VC}{\epsilon})$ and $\Theta(\tfrac{\VC}{\epsilon}\log(1/\epsilon))$, where
$\VC$ denotes the VC dimension of $\cH$ \cite{BHW89,EHKV89,Han16,Lar23}. Whereas the search problems
(learning) are well-understood, the sample sizes and algorithmic methods required for the associated
\emph{decision} problems (testing) are poorly understood.  

Indeed, \cref{thm:intro-informal} is the first tight bound (up to constant factors) for testing any
natural class of boolean functions in the distribution-free sample-based model, and juntas join
functions of support size $\leq k$ as the only classes with bounds known up to $\log(1/\epsilon)$
factors \cite{FH25}. The goal is to find a more general theory of testing to match the successful
theory of learning. \cref{thm:intro-informal} helps advance this goal by providing an important
contrast to the results on testing support size.  In particular, these results help clarify what we
think of as \emph{fine-grained} versions of the testing vs.~learning question, which we believe are
more insightful for understanding property testing and decision vs.~search in the distribution-free
sample-base model.

While \cref{question:testing-vs-learning} asks to compare the sample size of testing juntas
vs.~learning juntas, the obvious \cref{alg:obvious} for testing juntas does not even attempt to
learn the function itself, only the set of relevant variables. Therefore it seems more insightful to
compare junta testing to feature selection, and to the performance of \cref{alg:obvious}, than to
PAC learning. This comparison naturally generalizes to any other class $\cH$ via \cref{alg:erm}.  We
have shown that this algorithm is sample-optimal for testing juntas, whereas prior work on support
size \cite{GR16,FH25} shows that is not \emph{always} sample-optimal.  It is not clear to us what
general principle separates these examples.

\begin{boxquestion}
For which classes $\cH$ is \cref{alg:erm} sample-optimal for testing?
\end{boxquestion}

An important observation is that, unlike property testing models where the algorithm can make
queries, in the sample-based model \cref{alg:erm} is the unique 1-sided error tester; a 1-sided
error tester must never output $\REJECT$ if there exists $g \in \cH$ that is consistent with the
samples, whereas a 2-sided error algorithm must only make the correct decision with probability
$2/3$. So the question is equivalent to:

\begin{boxquestion}
For which classes $\cH$ is there an advantage for 2-sided error testers?
\end{boxquestion}

We may think of this question as asking when we can make a decision based on \emph{evidence} instead
of \emph{proof}; a 1-sided tester must not reject without \emph{proof}, whereas a 2-sided tester is
satisfied by strong evidence.  This is a fine-grained version of the testing vs.~learning question,
where we compare the \emph{algorithms} instead of only the sample sizes, and we believe it often
provides more insight into the decision vs.~search dichotomy. 

\subsection{Comparison to Other Models of Testing}
\label{section:intro-comparison}
As a byproduct of our analysis, we also get a tight lower bound (for constant $\epsilon$)
on \emph{testing junta trunction}. This was introduced recently by He \& Nadimpalli \cite{HN23} as an
instance of the general problem of testing truncation of distributions (see \eg
\cite{DNS23,DLNS24}), where the goal is to distinguish between samples from a distribution $\cD$ and
samples from the distribution $\cD$ truncated to some unknown set (in this case the satisfying
assignments of a junta).  This improves on the $\Omega(\log\binom{n}{k})$ lower bound of
\cite{HN23}.  See \cref{section:junta-truncation}.

\begin{boxtheorem}
\label{thm:intro-junta-truncation}
For sufficiently small constant $\epsilon > 0$, any tester for $k$-junta truncation must have sample
size at least
    \[
        \Omega\left(\sqrt{2^k \log\binom{n}{k}} + \log\binom{n}{k} \right)
    \]
\end{boxtheorem}

We have focused on testing properties of boolean functions in our discussion, but there are also
some tight or nearly-tight bounds for distribution-free sample-based testers for non-boolean
functions: \cite{RR22} gave tight bounds for one-sided error testing of subsequence-freeness of
strings; \cite{FY20} give almost tight bounds for testing linearity of real-valued functions.

There is a large body of work on testing juntas, and we will cite here only the optimal bounds in
each model.  If adaptive queries are allowed, optimal or nearly-optimal bounds of $\Theta(k \log k)$
are known for testers making adaptive queries, in both the uniform and distribution-free case
\cite{Bla09,Sag18,Bsho19}. If only non-adaptive queries are permitted, nearly optimal bounds of
$\widetilde \Theta(k^{3/2})$ are known for the uniform distribution \cite{Bla08,CSTWX18}, while
there is a lower bound of $\Omega(2^{k/3})$ for the distribution-free case \cite{LCSSX18} and an
upper bound of $O(2^k)$ via self-correctors \cite{HK07,AW12} (see \cite{LCSSX18}).

For tolerant testing, ignoring dependence on $\epsilon_1, \epsilon_2$, recent work \cite{NP24} gives
upper and lower bounds of $2^{\widetilde \Theta(\sqrt k)}$ for non-adaptive testers under the
uniform distribution, matching or improving earlier results for both adaptive and non-adaptive
testing.

Distribution-free sample-based junta testing is similar
to testing \emph{junta distributions}, a problem which has been studied recently in the field of
distribution testing \cite{ABR16,BCG19,CJLW21,beretta2025new}.

\subsection{Proof Overview}
\label{section:proof-overview}

The formal version of \cref{thm:intro-informal} is:

\begin{boxtheorem}
\label{thm:intro-main}
For any constant $\tau \in (0,1)$ and sufficiently large $n$, there exists a product distribution
$\cD$ over $\zo^n$ such that any sample-based $k$-junta tester for $k < (1-\tau)n/e$ with distance
parameter $\epsilon > 2^{-\tau n}$ requires 
\[
  \Theta\left(\frac{1}{\epsilon} \left( \sqrt{2^k \log\binom{n}{k}} + \log\binom{n}{k}\right)\right)
\]
samples. The upper bound holds for all $\epsilon > 0$ and is attained by the obvious algorithm.
The same bounds hold for $k$-feature selection.
\end{boxtheorem}

The full proofs are in \cref{section:upper,section:lb-testing}. We briefly describe them here.

\subsubsection{Upper Bounds}

Recent works \cite{GR16,BFH21,FH23,FH25} emphasize that distribution-free sample-based testing of
boolean functions is often best understood by a relation to \emph{distribution testing}, \ie testing
properties of distributions using samples (see \cite{Can20} for a survey on distribution testing).
To analyze the performance of \cref{alg:obvious}, we define a distribution testing task called
\emph{testing Supported on One-Per-Pair (SOPP)}, which turns out to be equivalent to junta
testing and feature selection.\\

\noindent
\textbf{Supported on One-Per-Pair (SOPP):} A probability distribution $p$ over $[2N]$ is supported
on one-per-pair if its support contains at most one element of each even--odd pair $p(2i), p(2i+1)$.
Testing SOPP is the task of distinguishing between distributions that are supported on one-per-pair
and distributions that are $\epsilon$-far in TV distance from being supported on one-per-pair, using
samples from the distribution. A one-sided error tester will reject a distribution only if it finds
a pair $(2i, 2i+1)$ where both elements are in the support of $p$.

\begin{boxlemma}
\label{lemma:intro-sopp-testing}
    Testing SOPP on $[2N]$ with one-sided error, distance $\epsilon > 0$, and success
    probability $1-\delta$, has sample size
    \[
        O\left( \frac{1}{\epsilon} \sqrt{N \log\frac{1}{\delta}}
            + \frac{1}{\epsilon} \log\frac{1}{\delta} \right) .
    \]
\end{boxlemma}
This bound is tight even for two-sided error testers: an improvement in the dependence on any
parameter would contradict our main lower bounds for testing $k$-juntas and $k$-feature selection.

Testing SOPP corresponds to testing whether an unknown function $f \colon \zo^n \to \zo$ is a junta
on a fixed set of variables $S$. For $x \in \zo^n$ and a subset $S \subseteq [n]$ of variables, we
write $x_S \in \zo^{|S|}$ as the values of $x$ on variables $S$. We take $N = 2^k$ and identify each
setting $z \in \zo^k$ of $k$ variables in $S$ with an even--odd pair $(2i, 2i+1)$. We define
distribution $p$ over $[2N]$ where $p(2i)$ is the probability that $f(\bm x) = 1$ when $\bm x$ is
chosen conditional on $\bm{x}_S = z$, while $p(2i+1)$ is the probability that $f(\bm x) = 0$.

We then obtain our upper bounds in \cref{thm:intro-main} for testing $k$-juntas and $k$-feature
selection by running the SOPP tester in parallel on all $\binom{n}{k}$ subsets of $k$ variables,
with error probability $\delta \approx \binom{n}{k}^{-1}$ to allow a union bound over all these
subsets.

\subsubsection{Testing \& Feature Selection Lower Bounds}

To prove our lower bounds on testing $k$-juntas and $k$-feature selection, we first prove lower
bounds for these tasks with constant distance parameter $\epsilon > 0$ under the uniform
distribution on $\zo^n$:

\begin{boxtheorem}
\label{thm:intro-unif}
For sufficiently small constant $\epsilon > 0$ and all $n,k \in \bN$ satisfying $k < n-2$, testing
$k$-juntas under the uniform distribution requires sample size at least.
    \[
        \Omega\left(\sqrt{2^k \log\binom{n}{k}} + \log\binom{n}{k} \right) .
    \]
\end{boxtheorem}
Similar bounds hold for $k$-feature selection. 

\paragraph{Comparison to the computational complexity of testing.}
We remark that for constant $\epsilon$ there is no
difference in sample complexity between distribution-free testing and testing under the uniform
distribution. This stands in contrast to the time complexity of testing. It is known that, under the strong exponential time hypothesis (SETH), testing $k$-juntas in the distribution-free setting is computationally harder than in the uniform distribution setting. The result in \cite{BKST23} implies that, assuming SETH, no algorithm can distribution-free test $1/3$-closeness to $k$-juntas in time $n^{k-\gamma}$ for any constant $\gamma>0$.\footnote{This result is implicit in \cite{BKST23} by combining their reduction with the fact that, under SETH, there is no constant factor approximation algorithm for the $k$-\textsc{SetCover} problem running in time $n^{k-\gamma}$ for any constant $\gamma>0$ \cite{KLM19}.} On the other hand, \cite{Val15} gives an algorithm running in time $n^{0.6 k}$ for learning $k$-juntas over the uniform distribution (which implies a testing algorithm with the same runtime). 

\paragraph*{Dependence on $\epsilon$.} Using this lower bound for the uniform distribution, we
obtain our main lower bounds in \cref{thm:intro-main} by constructing a fixed product distribution
$\mu$ over $\zo^n$, such that testing $k$-juntas under the uniform distribution reduces to testing
$(k+1)$-juntas under $\mu$. The reduction produces tester under the uniform distribution that uses
only on $\epsilon$ fraction of the number of samples as the tester under $\mu$, which gives us the
factor $1/\epsilon$ necessary for the tight lower bound in \cref{thm:intro-main}. The uniform
distribution cannot be used to get this tight bound because, under the uniform distribution, no two
$k$-juntas can have distance less than $2^{-k}$ (see \cref{section:lb-eps-unif}); achieving a lower
bound that holds for $\epsilon$ as small as $2^{-\Theta(n)}$ requires a distribution where the
juntas can be closer to each other.

\paragraph*{Proving \cref{thm:intro-unif}.} To prove the lower bound for the uniform distribution,
we identify $k$-juntas with a ``balls-and-bins'' process, as follows. We choose a collection $\cJ
\subseteq \binom{[n]}{k}$ of $k$-sets of variables, and a collection $\cF$ of balanced functions
$\zo^k \to \zo$, so that a balanced $k$-junta is obtained by choosing a set $S \in \cJ$ together
with a function $f \in \cF$ and taking the function
\[
    f_S \colon \zo^n \to \zo, \qquad f_S(x) \define f(x_S) .
\]
We think of each possible choice of $k$-junta $f_S$ as a ``ball'' (so there are $M =
|\cJ|\cdot|\cF|$ balls), and for any set $X \in \zo^n$ of $m$ sample points we think of the
labelling $f_S(X) \in \zo^m$ assigned to $X$ as the ``bin'' in which $f_S$ lands (so there are $N =
2^m$ bins). Our goal is to show that, with high probability over a random sample $\bm X$, the junta
``balls'' are nearly uniformly distributed among the ``bins''\!, \ie the labels $f_S(\bm X) \in
\zo^m$ assigned to $\bm X$ by choosing a random $k$-junta are nearly uniformly randomly distributed
over $\zo^m$. This would mean that $k$-juntas are indistinguishable from random functions. We prove
a simple ``balls-and-bins'' lemma (\cref{lemma:balls-and-bins}) which gives sufficient conditions
for $M$ random balls (not necessarily independent) thrown into $N$ bins to be nearly uniformly
distributed among the bins with high probability. Specifically, if the process satisfies
\begin{enumerate}[itemsep=0pt]
    \item \emph{Uniform collisions:} Conditional on balls $b_i, b_j$ colliding, they are distributed
        uniformly randomly among the bins; and
    \item \emph{Unlikely collisions:} On average, two randomly selected balls $b_{\bm i}, b_{\bm j}$
        will collide with probability at most $(1+o(1))\cdot \tfrac{1}{N}$,
\end{enumerate}
then the balls will be nearly uniformly distributed among the bins with high probability. The main
challenge in the lower bound is to choose a set of juntas which satisfy these properties
for large values of the sample size $m$ (\ie large numbers of bins $N = 2^m$). We use balanced
juntas because one can show that they always lead to uniform collisions. Now consider three options:
\begin{itemize}
    \item The easiest option is to take the set of parity functions on $k$ variables. 
        The parity functions on distinct sets of variables will collide with probability
        $2^{-m} = 1/N$ exactly, so the Unlikely Collisions property only requires that the number of
        bins is asymptotically smaller than the number of balls, \ie $2^m \ll \binom{n}{k}$,
        giving a lower bound of $\Omega(\log \binom{n}{k})$.
    \item Another natural choice is to take the set $\cF$ of all balanced functions $\zo^k \to
        \zo$ together with a set $\cJ$ of $n/k$ disjoint sets of variables. This is convenient
        because the balls associated with juntas on disjoint sets of variables are independent. To
        verify the Unlikely Collisions property, one must only consider the probability of collision
        of two juntas defined on the same set of variables. We do not include this analysis, since
        it leads to a suboptimal bound of $\Omega(\sqrt{2^k \log n})$.
    \item Our final choice is simply to take the set $\cF$ of all balanced functions together with
        the collection $\cJ$ of \emph{all} sets of $k$ variables. This leads to significant
        dependencies and the main challenge of the analysis is to handle these dependencies.
\end{itemize}
The idea in the analysis is to trade off between two quantities: juntas defined on sets of variables
$S, T$ which have large intersection are more likely to collide in the same bin, but large
intersections are less likely than small ones. We establish tail bounds on the probability of
collision as a function of the intersection size $|S \cap T|$, which are tight enough to trade off
against the probability of intersections of this size occurring. To establish these tail bounds, we
rely on the composition properties of negatively associated subgaussian and subexponential random
variables.

%% file: upper.tex
\section{Upper Bounds}
\label{section:upper}

We will prove tight upper bounds on distribution-free testing $k$-juntas and $k$-feature selection.
Our upper bounds will all follow from an upper bound on a distribution testing problem that we call
\emph{Supported on One Per Pair} (SOPP).

\subsection{Distribution Testing: Supported on One Per Pair}

Recent work on distribution-free sample-based property testing of Boolean functions has attempted to
relate these function testing problems to \emph{distribution testing} problems. We will phrase our
own upper bound this way, by defining a distribution testing problem and using it to solve our
function testing problem. This has the advantage of giving upper bounds for both $k$-junta testing
and $k$-feature selection, whereas upper bounds on either of these problems specifically do not
immediately translate into upper bounds for the other.

\newcommand{\msopp}{m^{\text{sopp}}}
\begin{boxdefinition}[Supported on One Per Pair (SOPP)]
  \label{def:sopp}
Let $p$ be a probability distribution over $[2N]$. We say $p$ is \emph{SOPP} if, for every $i$,
either $p(2i) = 0$ or $p(2i-1) = 0$. Write $\SOPP_N$ for the set of SOPP distributions over $[2N]$.
For any distribution $p$ over $[2N]$, we write
\[
  \|p-\SOPP_N\|_{\TV} \define \inf\{ \|p-q\|_{\TV} \;|\; q \in \SOPP_N \} \,.
\]
We say distribution $p$ over $[2N]$ is $\epsilon$-far from $\SOPP_N$ if $\|p-\SOPP_N\|_{\TV} \geq
\epsilon$.
\end{boxdefinition}

Calculating distance to SOPP is straightforward:
\begin{proposition}[Distance to SOPP]
  \label{prop:distance}
  Let $p$ be a distribution over $[2N]$. Then,
  \[
    \|p-\SOPP_N\|_{\TV} = \sum_{i\in [N]}\min\{p(2i), p(2i-1)\}.
  \]
\end{proposition}
\begin{proof}
Let $q$ be an SOPP distribution. Since $p$ and $q$ are probability distributions, we have 
\[
  \sum_{j\in[2N]:p(j)>q(j)} \big(p(j) - q(j)\big) = \sum_{j\in[2N]:p(j)\le q(j)} \big(q(j) - p(j)\big).
\]
Therefore:
\begin{align*}
    \|p-q\|_{\TV} &= \frac{1}{2}\sum_{j\in [2N]} |p(j)-q(j)|\\
    &=\frac{1}{2}\sum_{j\in[2N]:p(j)>q(j)} \big(p(j) - q(j)\big) + \frac{1}{2}\sum_{j\in[2N]:p(j)\le q(j)} \big(q(j) - p(j)\big)\\
    &=\sum_{j \in [2N] : p(j) > q(j)} (p(j)-q(j)) \\
    &\ge \sum_{i \in [N]} \left( p(2i) \ind{q(2i)=0} + p(2i-1) \ind{q(2i-1)=0} \right)\\
    &\ge \sum_{i \in [N]} \min\{ p(2i), p(2i-1) \}
\end{align*}
and equality is achieved exactly when for each $i \in [N]$, we choose $q(j) = 0$ on the element $j
\in \{2i, 2i-1\}$ minimizing $p(j)$, and $q(j') \geq p(j')$ for the opposite $j'$ in the pair.
\end{proof}

\begin{definition}[SOPP Testing]
An algorithm $A$ is a (one-sided) \emph{SOPP tester} with sample complexity $m =
m(N,\epsilon,\delta)$ if, given any parameters $N \in \bN$, $\epsilon, \delta \in (0,1)$, and sample
access to any distribution $p$ over $[2N]$, $A$ will take at most $m$ independent random samples
$\bm S$ from distribution $p$ and output the following:
\begin{enumerate}
\item If $p \in \SOPP_N$ then $\Pru{\bm S}{A(\bm S) \text{ outputs } \ACCEPT} = 1$; and
\item If $p$ is $\epsilon$-far from $\SOPP_N$ then $\Pru{\bm S}{A(\bm S) \text{ outputs } \REJECT} \geq 1-\delta$.
\end{enumerate}
We write $\msopp(N,\epsilon,\delta)$ for the optimal sample complexity of a (one-sided) SOPP tester
given parameters $N, \epsilon, \delta$.
\end{definition}

Testing $k$-juntas and $k$-feature selection both reduce to testing SOPP with small error $\delta <
\binom{n}{k}^{-1}$:

\newcommand{\bin}{\mathsf{bin}}
\begin{boxlemma}
\label{res:ub-reduction-to-sopp}
The sample complexity of one-sided distribution-free testing $\cJ_{k,n}$ with error probability
$\delta$ is at most $\msopp(2^k, \epsilon/2, \delta \binom{n}{k}^{-1})$. The sample complexity of
distribution-free $k$-feature selection with error probability $\delta$ is also at most $\msopp(2^k,
\epsilon/2, \delta \binom{n}{k}^{-1})$.
\end{boxlemma}
\begin{proof}
For a subset $S \subset [n]$ and binary string $x \in \zo^n$, we write $x_S \in \zo^{|S|}$
for the subsequence of $x$ on coordinates $S$. To design the algorithms we require some definitions.

On input function $f \colon \zo^n \to \zo$ and distribution $p$ over $\zo^n$, define the following
distributions.
For each set $S \in \binom{[n]}{k}$ of $k$ variables, define a distribution $p_S$
over $[2N]$ with $N=2^k$, where each $i \in [2N]$ has probability
\[
  p_S(i) \define \sum_{x \in \zo^n} p(x) \ind{(x_S, f(x))  = \bin(i-1)} \,,
\]
where $\bin(i)$ denotes the $(k+1)$-bit binary representation of $i \in [2N] = [2^{k+1}]$. Observe
that:
\begin{enumerate}
\item If $f$ is a $k$-junta, defined on relevant variables $S \in \binom{[n]}{k}$, then
$p_S \in \SOPP_N$.
\item If $f$ is $\epsilon$-far from all $k$-juntas on variables $S \in \binom{[n]}{k}$, then $p_S$
is $\epsilon/2$-far from $\SOPP_N$. Otherwise, if $\|p_S-\SOPP_N\|_{\TV} \leq \epsilon/2$, then
$f$ is $\epsilon$-close to the function $g$ on variables $S$ defined by
\[
  g(x) \define \argmax_{b \in \zo} \sum_{z \in \zo^n} p(z) \ind{ z_S = x_S \wedge f(z) = b }
\]
Equivalently, $g$ is defined as the $k$-junta with relevant variables $S$ that is closest to $f$
under $p$. Indeed, we have by construction that for every $x\in\zo^n$, 
$$
\sum_{z\in\zo^n} p(z)\ind{z_S=x_S\wedge f(z)\neq g(z)} = \min\{p_S(2i),p_S(2i-1)\}
$$
where $i$ is such that $\{(x_S,1),(x_S,0)\} = \{\bin(2i),\bin(2i-1)\}$. Applying \Cref{prop:distance}, we have
\begin{align*}
    \Pru{\bm{x}\sim p}{f(\bm{x})\neq g(\bm{x})} &= \sum_{i\in [2N]} \min\{p_S(2i),p_S(2i-1)\} \\
    &= \|p_S- \SOPP_N\|_{\TV} \leq \epsilon/2.
\end{align*}
\item For $\bm x \sim p$, the random variable $(\bm{x}_S, f(\bm x))$ is distributed as a sample from
$p_S$.
\end{enumerate}
Then our tester is as follows:
\begin{enumerate}
\item Sample $m = \msopp(2^k, \epsilon, \delta \binom{n}{k}^{-1})$ labeled points $\bm{S}_f = \{
(\bm{x}_i, f(\bm{x}_i)) \;|\; i \in [m]\}$.
\item Run the testers for SOPP on each $p_S$ in parallel using the samples $\bm{T}_S
\define \{ ((\bm{x}_i)_S, f(\bm{x}_i)) \;|\; i \in [m]\}$.
\item Output $\REJECT$ if \emph{all} of these testers output $\REJECT$, otherwise $\ACCEPT$.
\end{enumerate}
The algorithm for $k$-feature selection is similar, except that in the last step it outputs an
arbitrary set $S$ for which the SOPP tester on $p_S$ did not reject.

If $f$ is a $k$-junta, then there is $S \in \binom{[n]}{k}$ such that $p_S \in \SOPP_N$, so the
probability that the tester rejects is at most $\delta \binom{n}{k}^{-1}$; if the tester for
$\SOPP_N$ has one-sided error, then this probability is 0. If $f$ is $\epsilon$-far
from being a $k$-junta, then every $p_S$ is $\epsilon/2$-far from $\SOPP_N$, so, by the union bound,
the probability the tester fails to output $\REJECT$ is at most $\binom{n}{k} \cdot \delta
\binom{n}{k}^{-1} = \delta$. A similar argument shows that the $k$-feature selection algorithm
succeeds.
\end{proof}

\subsection{Upper Bound on Testing SOPP}

Together with the reduction in \cref{res:ub-reduction-to-sopp}, the following lemma immediately
implies our upper bounds in \cref{thm:intro-main}.

\begin{boxlemma}
\label{res:ub-sopp}
For every $N \in \bN$ and $\epsilon, \delta \in (0,1)$, the sample complexity of testing $\SOPP_N$
is at most
\[
  \msopp(N, \epsilon, \delta) = O\left( \tfrac{1}{\epsilon} \sqrt{N \log(1/\delta)} +
\tfrac{1}{\epsilon} \log(1/\delta) \right) .
\]
\end{boxlemma}
\begin{remark}
The bound in this lemma is tight up to constant factors (even if one allows two-sided error
testers): if there was an improvement in the dependence on any of the parameters $N, \epsilon$, or
$\delta$, then it would contradict our lower bounds for testing $k$-juntas.
\end{remark}
\begin{proof}
The tester is the natural one: on input distribution $p$ over $[2N]$, take a sample $\bm S = \{
\bm{x}_1, \dotsc, \bm{x}_{2m} \}$ of size $2m$ and output $\REJECT$ if and only if there exists $i \in
[N]$ such that $\{2i, 2i-1\} \subseteq  \bm S$.  We will choose
\[
  m \define \tfrac{1}{\epsilon}( \sqrt{32 N \ln(1/\delta)} + 32 \ln(1/\delta)) \,.
\]
To prove correctness of this tester, it suffices to show that it will
output $\REJECT$ with probability at least $1-\delta$ when $p$ is $\epsilon$-far from $\SOPP_N$.
Hereafter we assume $\epsilon^* \define \|p- \SOPP_N\|_{\TV}$ satisfies $\epsilon^* \geq \epsilon$.

For this proof, it will be convenient to treat $p$ as a vector in $\bR^{2N}$. We may assume without loss of generality that $p_{2i} \leq p_{2i-1}$ for all $i \in [N]$. Furthermore, we define $q, r \in \bR^{2N}$ as
\begin{align*}
  \forall i \in [N] \;\qquad\; q_{2i} &\define p_{2i}, &r_{2i} &\define 0 \\
                          q_{2i-1} &\define 0,      &r_{2i-1} &\define p_{2i-1} \,.
\end{align*}
\newcommand{\cover}{\mathsf{cover}}
Observe that $q_{2i} \leq r_{2i-1}$ for all $i \in [N]$, and $\epsilon^* = \sum_{i\in [2N]} q_i$. We will say
a set $S \subset [2N]$ \emph{covers} mass $\rho$ of $q$ if the total $q$-mass of $S$ is $\rho$, \ie 
\[
  \cover(q,S) \define \sum_{i=1}^{2N} q_i \ind{i \in S } = \rho \,.
\]
We partition $\bm S = \bm{S}_1 \cup \bm{S}_2$ arbitrarily into two subsets of size $m$.
First we show that if $\bm{S}_1$ covers large mass of $q$ then the tester will reject with high
probability:

\begin{claim}
Suppose that
\[
  \cover(q, \bm{S}_1) \geq \min\left(\frac{\epsilon}{2}, \frac{\epsilon^2 m}{32 N} \right) \,.
\]
Then the probability that the tester outputs $\REJECT$ is at least $1-\delta/2$.
\end{claim}
\begin{proof}[Proof of claim]
Consider any sample point $\bm{x} \in \bm{S}_2$. The probability that there exists $i \in [N]$
such that $\{2i, 2i-1\} \subseteq  \bm S_1$ (which causes the tester to output
$\REJECT$) is at least
\[
  \sum_{i=1}^{N} r_{2i-1} \ind{2i \in \bm{S}_1}
  \geq \sum_{i=1}^N q_{2i} \ind{2i \in \bm{S}_1}
  = \cover(q, \bm{S}_1) \,.
\]
Since each sample point $\bm{x} \in \bm{S}_2$ is independent, the probability that the tester fails
to output $\REJECT$ is at most
\[
  \left(1 - \cover(q, \bm{S}_1)\right)^m \leq e^{-m \cdot \cover(q, \bm{S}_1)}
  \leq e^{-m \cdot \min\left(\tfrac{\epsilon}{2}, \tfrac{\epsilon^2 m}{32 N} \right)} \,.
\]
Since $m > 2 \tfrac{1}{\epsilon} \ln(2/\delta)$ and $m > \tfrac{1}{\epsilon} \sqrt{32 N
\ln(2/\delta)}$, this is at most $\delta/2$.
\end{proof}

\begin{claim}
With probability at least $1-\delta/2$ over $\bm{S}_1$,
\[
  \cover(q, \bm{S}_1) \geq \min\left( \frac{\epsilon}{2}, \frac{\epsilon^2 m}{32 N}
\right) \,.
\]
\end{claim}
\begin{proof}[Proof of claim]
Write $\bm{S}_1 = \{ \bm{x}_1, \dotsc, \bm{x}_m \}$ and for each $j \in [m]$ write $\bm{T}_j \define
\{ \bm{x}_1, \dotsc, \bm{x}_j \}$ for the first $j$ sample points in $\bm{S}_1$.

For each $j \in [m]$, let $\bm{X}_j \in \zo$ take value 1 if and only if either
$\cover(q,\bm{T}_{j-1}) \geq \epsilon/2$, or the sample point $\bm{x}_j$ covers at least
$\epsilon/4N$ previously uncovered mass of $q$, \ie
\[
  \cover(q, \bm{T}_j) \geq \cover(q, \bm{T}_{j-1}) + \frac{\epsilon}{4N} \,.
\]
Observe that for each $j \in [m]$, either $\cover(q,\bm{T}_{j-1}) \geq \epsilon/2$, or
\begin{align*}
\Pruc{}{ \bm{X}_j = 1 }{ \cover(q,\bm{T}_{j-1}) < \epsilon/2 }
  &= \sum_{i=1}^{2N} q_i \ind{q_i \geq \epsilon/4N} \left( 1 - \ind{i \in \bm{T}_{j-1}} \right) \\
  &\geq \sum_{i=1}^{2N} q_i \ind{q_i \geq \epsilon/4N} - \sum_{i=1}^{2N} q_i \ind{i \in \bm{T}_{j-1} } \\
  &\geq \sum_{i=1}^{2N} q_i \ind{q_i \geq \epsilon/4N} - \frac{\epsilon}{2} \,.
\end{align*}
Since
\[
  \sum_{i=1}^{2N} q_i \ind{ q_i < \epsilon/4N} < N \cdot \frac{\epsilon}{4N} = \epsilon/4 \,,
\]
we have for each $j$,
\[
  \Pr{ \bm{X}_j = 1 } \ge \Pr{ \cover(q,\bm{T}_{j-1}) \geq \epsilon/2 } + \Pr{ \cover(q,
\bm{T}_{j-1}) < \epsilon/2} \cdot \left( \sum_{i=1}^{2N} q_i - \frac{\epsilon}{4} - \frac{\epsilon}{2}
\right) \geq
\frac{\epsilon}{4} \,,
\]
so $\Ex{\sum_{j=1}^m \bm{X}_j} \geq \epsilon m / 4$.  If $\sum_{j=1}^m \bm{X}_j > \epsilon m / 8$
then either $\cover(q, \bm{S}_1) \geq \epsilon/2$ or $\cover(q, \bm{S}_1) \geq \tfrac{\epsilon}{4N}
\cdot \tfrac{\epsilon m}{8}$. So
\begin{align*}
  \Pr{ \mathsf{cover}(q, \bm{S}_1) < \min\left(\frac{\epsilon}{2}, \frac{\epsilon^2 m}{32N}\right) }
  &\leq \Pr{ \sum_{j=1}^m \bm{X}_j \leq \frac{\epsilon m}{8} } \,.
\end{align*}
The random variables $\bm{X}_j$ are not independent, but they take value $\bm{X}_j = 1$ with
probability at least $\epsilon/4$ regardless of the value of $\bm{X}_1, \dotsc, \bm{X}_{j-1}$, so
for every threshold $t$ we may upper bound $\Pr{ \sum_{j=1}^m \bm{X}_j < t }$ as if each $\bm{X}_j$
was an independent Bernoulli random variable with parameter $\epsilon/4$.  By the Chernoff bound, we
get
\[
  \Pr{ \sum_{j=1}^m \bm{X}_j < \frac{\epsilon m}{8} }
  \leq \Pr{ \sum_{j=1}^m \bm{X}_j < \frac{1}{2} \Ex{\sum_{j=1}^m \bm{X}_j} }
  \leq e^{-\frac{\epsilon m}{32}} < \delta/2 \,,
\]
where the last inequality is because $m > 32 \tfrac{1}{\epsilon} \ln(2/\delta)$.
\end{proof}
Taking a union bound over the two failure probabilities of $\delta/2$ concludes the proof.
\end{proof}

%% file: lower.tex
\section{Lower Bounds for Testing and Feature Selection}
\label{section:lb-testing}

We start with lower bounds for the uniform distribution.

\begin{boxtheorem}[Lower Bound for the Uniform Distribution (restatement of \cref{thm:intro-unif})]
\label{thm:lb-testing-unif}
Let $n, k \in \bN$ satisfy $k \leq n/e$, and $\varepsilon > 0$ be a sufficiently small constant. Then,
any $k$-junta tester under the uniform distribution on $\zo^n$ with distance parameter $\varepsilon$
requires sample size at least 
\[
  \Omega\left( \sqrt{2^k \log\binom{n}{k}} + \log\binom{n}{k} \right) .
\]
The same lower bound holds for $k$-feature selection.
\end{boxtheorem}

\subsection{A Balls \& Bins Lemma}

Our lower bounds will be achieved by associating $k$-juntas with balls that are thrown into bins
according to the uniformly random sample drawn from the uniform distribution. We will need a lemma
about the uniformity of $M$ balls thrown into $N$ bins.

Suppose we have $M$ balls which are thrown into $N$ bins according to some random process, \ie let
the $M$ balls be random variables $( \bm{\beta}_1, \dotsc, \bm{\beta}_M )$ taking values in $[N]$.
The balls may not be independent. For each bin $\ell \in [N]$, we define $\bm{B}_\ell$ as the
number of balls landing in bin $\ell$,
\[
  \bm{B}_\ell \define \sum_{i=1}^M \ind{\bm{\beta}_i = \ell} \,.
\]
We are interested in the probability that the balls are nearly evenly distributed. Specifically, the
placement of balls into bins creates a probability distribution where element $\ell \in [N]$ is
assigned probability density $\bm{B}_\ell / M$, and we want this distribution to be close to
uniform. We want an upper bound on
\[
  \Pr{ \sum_{\ell=1}^N \left| \frac{\bm{B}_\ell}{M} - \frac{1}{N} \right| > \epsilon } .
\]

\begin{boxlemma}
\label{lemma:balls-and-bins}
Suppose $M$ balls $\bm{\beta}_1, \dotsc, \bm{\beta}_M$ are thrown into $N$ bins, with the balls
satisfying the conditions:
\begin{itemize}
  \item (Uniform Collisions). For every bin $\ell \in [N]$, and every $i,j \in [M]$,
  \[
    \Pruc{}{\bm{\beta}_i = \bm{\beta}_j = \ell }{ \bm{\beta}_i = \bm{\beta}_j } = \frac{1}{N} ,
  \]
  \item (Unlikely Collisions). For uniformly random $\bm{i}, \bm{j} \sim [M]$,
  \[
    \Exu{\bm i, \bm j}{ \Pr{\bm{\beta_i} = \bm{\beta_j} }} = (1+o(1)) \frac{1}{N} \,,
  \]
  where the $o(1)$ term is with respect to $N \to \infty$.
\end{itemize}
Then for every constant $\epsilon > 0$,
\[
  \Pr{ \sum_{\ell = 1}^N \left| \frac{\bm{B}_\ell}{M} - \frac{1}{N} \right| < \epsilon }
    \geq 1 - o(1) \,.
\]
\end{boxlemma}
\begin{proof}
We begin with the following claim, which we prove below.

\begin{claim}
\label{claim:balls-and-bins}
For every bin $\ell \in [N]$ and every constant $\epsilon > 0$,
\[
  \Pr{ \left| \frac{M}{N} - \bm{B}_\ell \right| > \epsilon \frac{M}{N} } = o(1) \,.
\]
\end{claim}

With this claim, we complete the proof of \cref{lemma:balls-and-bins} as follows. For a fixed
allocation of balls, we say bin $\ell$ is ``good'' if $\left| \tfrac{M}{N} - B_\ell \right| \leq
\tfrac{\epsilon}{4} \tfrac{M}{N}$. By \cref{claim:balls-and-bins}, the expected number of bad bins
is $o(N)$. Using Markov's inequality, with probabiility at least $1-o(1)$ there will be at most
$o(N)$ bad bins. Suppose this event occurs. Then the number of balls in good bins is at least
\[
(1-o(1))N \cdot (1 - \epsilon/4)\tfrac{M}{N} > (1-\epsilon/2)M ,
\]
and the number of balls in bad bins is therefore at most $\epsilon M / 2$. Then
\begin{align*}
  \sum_{\ell \in [N]} \left| \frac{M}{N} - B_\ell \right|
  &= \sum_{\ell \text{ good}} \left| \frac{M}{N} - B_\ell \right|
     + \sum_{\ell \text{ bad}} \left| \frac{M}{N} - B_\ell \right| \\
  &\leq \frac{\epsilon}{4} M + \sum_{\ell \text{ bad}} \left(\frac{M}{N} + B_\ell \right) \\
  &\leq \frac{\epsilon}{4} M + o(N) \cdot \frac{M}{N} + \frac{\epsilon}{2} M < \epsilon M .
\end{align*}

It remains to prove the claim.

\begin{proof}[Proof of \cref{claim:balls-and-bins}]
Fix any bin $\ell \in [N]$ and constant $\epsilon > 0$.  We will use Chebyshev's inequality, so we must compute
\begin{align*}
  \Ex{ \left( \frac{M}{N} - \bm{B}_\ell \right)^2 }
  &= \Ex{\bm{B}_\ell^2} - \frac{M^2}{N^2} 
\end{align*}
where equality holds since $\Ex{\bm{\beta}_\ell} = M/N$ by uniform collisions. Now, using both the conditions of uniform and unlikely collisions,
\begin{align*}
\Ex{\bm{B}_\ell^2}
&= \Ex{ \sum_{i, j \in [M]} \ind{\bm{\beta}_i = \bm{\beta}_j }
    \cdot \ind{\bm{\beta}_i = \bm{\beta}_j = \ell } } 
= \sum_{i,j} \Pr{ \bm{\beta}_i = \bm{\beta}_j }
  \cdot \Pruc{}{\bm{\beta}_i = \bm{\beta}_j = \ell }{ \bm{\beta}_i = \bm{\beta}_j } \\
&= \frac{1}{N} \sum_{i,j} \Pr{ \bm{\beta}_i = \bm{\beta}_j }
  \tag{Uniform collisions} \\
&= \frac{M^2}{N} \Exu{\bm i, \bm j \sim [M]}{\Pr{ \bm{\beta}_{\bm i} = \bm{\beta}_{\bm j} }} \\
&= \frac{M^2}{N^2} (1 + o(1)) \tag{Unlikely collisions}.
\end{align*}
Now
\[
  \Ex{ \left( \frac{M}{N} - \bm{B}_\ell \right)^2 } = o\left(\frac{M^2}{N^2}\right) \,,
\]
so by Chebyshev's inequality,
\[
  \Pr{\left| \frac{M}{N} - \bm{B}_\ell \right| > \epsilon \frac{M}{N} }
  \leq \Ex{\left(\frac{M}{N}-\bm{B}_\ell\right)^2} \cdot \frac{N^2}{\epsilon^2 M^2}
  = o(1) \,,
\]
since $\epsilon > 0$ is constant.
\end{proof}
\end{proof}

\begin{example}
\label{ex:pairwise-indep}
If we have $M = \omega(N)$ uniform and pairwise independent balls, then the uniform collision condition is
trivially satisfied, and the unlikely collision condition is satisfied because
\begin{align*}
  \Exu{\bm i, \bm j}{ \Pr{\bm{\beta_i} = \bm{\beta_j}} }
  &= \Pr{\bm i = \bm j} + \Pr{ \bm i \neq \bm j} \frac{1}{N}
  = \frac{1}{M} + \left(1 - \frac{1}{M}\right) \frac{1}{N} \\
  &= \frac{1}{N} \left( 1 + \frac{N-1}{M} \right)
  \leq \frac{1}{N}(1 + o(1)) \,.
\end{align*}
\end{example}

\subsection{Balanced Junta Setups}

\begin{boxdefinition}[Balanced $k$-Junta Setup]
A \emph{$k$-junta setup} on $\zo^n$ is a pair $(\cJ, \cF)$ where
$\cJ \subseteq \binom{[n]}{k}$ is a collection of $k$-subsets of $[n]$, and $\cF$ is a
collection of balanced functions $\zo^k \to \zo$. 
\end{boxdefinition}

We associate a balanced $k$-junta setup with a probability distribution and a balls-and-bins
process. First, we require some notation.  For any set $S \in \cJ$ and any $x \in \zo^n$, we write $x_S \in
\zo^k$ for the substring of $x$ on the coordinates $S$. For any function $f : \zo^k \to \zo$, we
write $f_S : \zo^n \to \zo$ for the function 
\[
  f_S(x) \define f(x_S) \,.
\]
For any sequence $X = (x_1, \dotsc, x_m) \in (\zo^n)^m$, we define
\[
  f_S(X) \define (f_S(x_1), \dotsc, f_S(x_m)) \,.
\]

\paragraph*{Distribution over $k$-juntas.}
We write $\cD(\cJ,\cF)$ for the distribution over balanced $k$-juntas obtained by choosing a
uniformly random set $\bm S \sim \cJ$ of variables and a uniformly random function $\bm f \colon
\zo^k \to \zo$ from $\cF$, and then taking $\bm{f}_{\bm S} \colon \zo^n \to \zo$.

\paragraph*{Balls-and-bins process.}
For any sample-size parameter $m$, we associate the $k$-junta setup $(\cJ, \cF)$ with the following
balls-and-bins process. We have $N = 2^m$ bins indexed by the binary strings $\zo^m$, and we have $M
= |\cJ| \binom{2^k}{2^{k-1}}$ balls, indexed by pairs $(S, f)$ where $S \in \cJ$ is a set of $k$
variables and $f : \zo^k \to \zo$ is a balanced function on $k$ variables.

The $M$ balls are assigned to bins according to the following process. We choose a random sequence
$\bm{X} = (\bm{x}_1, \bm{x}_2, \dotsc, \bm{x}_m)$ of $m$ independent and uniformly random strings
$\bm{x}_i \sim \zo^n$ for $i \in [m]$. We then assign ball $(S,f)$ to bin $f_S(\bm X)$.

As in \cref{lemma:balls-and-bins}, we write $\bm{B}_\ell$ for the number of balls (balanced juntas)
assigned to bin $\ell$ (\ie the number of juntas which assign label $\ell \in \zo^m$ to the sample
points $\bm X$).

\begin{observation}
\label{obs:lb-junta-balls}
For any sequence $X = (x_1, \dotsc, x_m)$ and a random $\bm{f}_{\bm S} \sim \cD(\cJ, \cF)$, the
label $\bm{f}_{\bm S}(X) \in \zo^m$ is distributed identically to the bin $\bm \ell \in \zo^m$ that contains the ball $(\bm S, \bm f)$.
\end{observation}

By applying \cref{lemma:balls-and-bins} to the $k$-junta setups, we see that if the parameter $m$ is
small enough that that the balls are nearly uniformly distributed in the bins, then the labels of
$\bm X$ given by a random junta are indistinguishable from uniform.

\begin{boxproposition}
\label{prop:lb-junta-setups}
Let $\epsilon > 0$ be any constant and let $n$ be sufficiently large.  Let $(\cJ, \cF)$ be a
balanced $k$-junta setup on $\zo^n$, and let $m$ be any parameter such that the associated
balls-and-bins process satisfies the Uniform and Unlikely Collisions conditions of
\cref{lemma:balls-and-bins}. Then
\[
  \Pru{\bm X}{ \| \unif(\zo^m) - \bm{f}_{\bm S}(\bm X) \|_\TV > \epsilon } < 1/100 ,
\]
where $\bm X = (\bm x_1, \dotsc, \bm x_m)$ is a sequence with each $\bm{x}_i$ distributed i.i.d as
$\bm{x}_i \sim \unif(\zo^n)$, and $\bm{f}_{\bm S}(\bm X)$ is the distribution over labels obtained by
choosing $\bm{f}_{\bm S} \sim \cD(\cJ, \cF)$.
\end{boxproposition}
\begin{proof}
Write $N = 2^m$ for the number of bins, so that the uniform distribution over $\zo^m$ assigns
probability $1/N$ to each bin. For a fixed sequence $X = (x_1, \dotsc, x_m)$ of samples and a fixed
bin $\ell \in \zo^m$, \cref{obs:lb-junta-balls} implies that $\Pr{ \bm{f}_{\bm S}(X) = \ell } =
B_\ell/N$ where $B_\ell$ is the number of juntas (balls) assigned to bin $\ell$. The conclusion holds
by \cref{lemma:balls-and-bins}.
\end{proof}

\begin{example}[Parities]
\label{ex:parities}
If we let $\cJ = \binom{n}{k}$ be the set of all $k$-sets of variables, and let $\cF$ be the
singleton set containing only the parity function $f \colon \zo^k \to \zo$ defined as $f(x) \define
\bigoplus_i x_i$, then it is easy to check that the resulting balls-and-bins process has pairwise
independent balls. By the calculation in \cref{ex:pairwise-indep}, it suffices to take $M =
\omega(N)$; in other words, we have $M = \binom{n}{k}$ balls so to apply \cref{prop:lb-junta-setups}
it suffices to take $m = \frac{1}{2} \log\binom{n}{k}$. Applying this calculation with
\cref{prop:lb-junta-setups} (via \cref{lemma:lb-k-junta,lemma:lb-k-feature} below), we obtain the
$\Omega(\log\binom{n}{k})$ term in \cref{thm:lb-testing-unif}.
\end{example}

It is intuitively clear that if $m$ samples are insufficient to distinguish the labels given by the
juntas from uniformly random labels, then $m$ should be a lower bound on testing $k$-juntas and
$k$-feature selection. We formalize these arguments in the appendix
(\cref{section:missing-proofs-lb-testing}), and state the resulting technical lemmas here:

\begin{boxlemma}
\label{lemma:lb-k-junta}
Let $\epsilon \in (0,1/4)$ be any constant, let $n$ be sufficiently large, and let $k = k(n) < k-2$.
Let $(\cJ, \cF)$ be a balanced $k$-junta setup on $\zo^n$, and let $m = m(n,k,\epsilon)$ be any
sample size parameter such that the associated balls-and-bins process satisfies the Uniform and
Unlikely Collisions conditions of \cref{lemma:balls-and-bins}. Then $m$ samples is insufficient for
a sample-based $k$-junta tester with distance parameter $\epsilon$ under the uniform distribution.
\end{boxlemma}

In the next lemma, we simplify the proof by requiring that random juntas from $(\cJ, \cF)$ are far
from being $(k/2)$-juntas. This will be the case for the junta setups that we use; see
\cref{prop:lb-feature-counting}.  Note that we use the junta setup on $n$ bits but the lower bound
is for feature selection on domain $\zo^{2n}$ with $2n$ bits; this is again just to simplify the
proof.

\begin{boxlemma}
\label{lemma:lb-k-feature}
Let $\epsilon \in (0,1/4)$ be any constant, let $n$ be sufficiently large, and let $k=k(n) < n-2$.
Let $(\cJ, \cF)$ be a balanced $k$-junta setup on $\zo^n$ such that a uniformly random $f \sim \cF$
is $\epsilon$-far from every $(k/2)$-junta with probability at least $9/10$, and let $m =
m(n,k,\epsilon)$ be any sample size parameter such that the associated balls-and-bins process
satisfies the Uniform and Unlikely Collisions conditions of \cref{lemma:balls-and-bins}. Then $m$
samples is insufficient for a sample-based $k$-feature selector with distance parameter $\epsilon$
under the uniform distribution on $\zo^{2n}$.
\end{boxlemma}

Since the \cref{ex:parities} handles the $\Omega(\log\binom{n}{k})$ term, the proof of the lower
bound for the uniform distribution, \cref{thm:lb-testing-unif}, will be complete once we establish
the Uniform Collisions property (in the next section) and the Unlikely Collisions property with $m =
o(\sqrt{2^k\log\binom{n}{k}})$, which is finally accomplished in \cref{lemma:lb-unlikely-collisions}
below.

\subsubsection{Balanced junta setups satisfy Uniform Collisions}

For fixed $S, T \in \binom{[n]}{k}$, fixed $f, g : \zo^k \to \zo$ and a fixed $z \in
\zo^*$ of length $\Delta \coloneqq |S \cap T|$, we define $\rho_f(S, z)$ as the probability of completing $z$ into a
$k$-bit string where $f(\cdot)=1$,
\ie
\begin{equation}
\label{eq:lb-rho-definition}
\begin{aligned}
  \rho_f(S, z) &\define \Pruc{\bm{x} \sim \zo^n}{ f_S(\bm x) = 1 }{ \bm{x}_{S \cap T} = z} \\
  \rho_g(T, z) &\define \Pruc{\bm{x} \sim \zo^n}{ g_T(\bm x) = 1 }{ \bm{x}_{S \cap T} = z}
\end{aligned}
\end{equation}
Since $f, g$ are both balanced functions, we have for all $S, T$ that
\begin{equation} \label{eq:expectation of rho}
  \Exu{\bm z \sim \zo^\Delta}{ \rho_f(S, \bm z) } = \Exu{\bm z \sim \zo^\Delta}{\rho_g(T, \bm z)} = \frac{1}{2} \,.
\end{equation}
For fixed $z \in \zo^\Delta$, we may write
\begin{equation}
\label{eq:lb-rho-sum}
  \rho_f(S,z) = \frac{1}{2^{k-\Delta}} \sum_{w \in \zo^{k-\Delta}} \ind{ f_S \text{ takes value 1 on
the combination of $z,w$}}
\end{equation}
where we mean that $f_S$ is given an input taking values $z$ on $S \cap T$ and $w$ on the remaining
$k-\Delta$ bits of $S$.

\begin{boxproposition}[Uniform Collisions]
Let $(\cJ, \cF)$ be any balanced $k$-junta setup. Then the associated balls-and-bins process
satisfies the Uniform Collisions condition in \cref{lemma:balls-and-bins}.
\end{boxproposition}
\begin{proof}
Since each $\bm{x}_i \in \bm X$ is independent, it suffices to show that for uniformly random $\bm x
\sim \zo^n$,
\[
\Pruc{\bm x \sim \zo^n}{ f_S(\bm x) = g_T(\bm x) = \ell_i }{ f_S(\bm x) = g_T(\bm x)} = 1/2 .
\]
Write $Z \define S \cap T$.  Since $S \setminus Z$ and $T \setminus Z$ are disjoint, we can write
$f_S(\bm x) = f(\bm z, \bm{y}_1)$ and $g_T(\bm x) = g(\bm z, \bm{y}_2)$ where $\bm z \sim
\zo^\Delta$ and $\bm{y}_1, \bm{y}_2 \sim \zo^{k-\Delta}$ are independent. So we want to show
\[
\Pruc{}{ f(\bm z, \bm{y}_1) = g(\bm z, \bm{y}_2) = \ell_i }{ f(\bm z, \bm{y}_1) = g(\bm z, \bm{y}_2)} = 1/2 .
\]
Under the condition, the pair $\left( (\bm z, \bm{y}_1), (\bm z, \bm{y}_2) \right)$ is drawn
uniformly from the set of pairs $\left( (z, y_1), (z, y_2) \right)$ which satisfy $f(z,y_1) = g(z,
y_2)$. We need to show that there are an equal number of these pairs where $f(z,y_1) = g(z, y_2) = 1$ and where
$f(z,y_1) = g(z, y_2) = 0$. The number of pairs where $f(z,y_1) = g(z, y_2) = 1$ is
\[
  \sum_{z \in \zo^\Delta} 2^{2(n - \Delta)} \rho_f(S, z) \rho_g(T, z), 
\]
and the number of pairs where $f(z,y_1) = g(z, y_2) = 0$ is
\begin{align*}
  &\sum_{z \in \zo^\Delta} 2^{2(n - \Delta)} (1-\rho_f(S, z))(1- \rho_g(T, z)) \\
  &\qquad= 2^{2(n - \Delta)} \sum_{z \in \zo^\Delta}
    \left( 1 - \rho_f(S, z) - \rho_g(T,z) + \rho_f(S,z) \rho_g(T,z) \right) \\
  &\qquad= 2^{2n - \Delta} - 2^{2n - \Delta} \Exu{\bm z}{ \rho_f(S, \bm z) } - 2^{2n - \Delta} \Exu{\bm z}{
\rho_g(T, \bm z)} + 2^{2(n - \Delta)} \sum_{z \in \zo^\Delta} \rho_f(S,z) \rho_g(T,z) \\
  &\qquad= 2^{2(n - \Delta)} \sum_{z \in \zo^\Delta} \rho_f(S,z) \rho_g(T,z) \,,
\end{align*}
where we have used \Cref{eq:expectation of rho}.
Therefore the
number of 1-valued pairs is equal to the number of 0-valued pairs, which completes the proof.
\end{proof}

\subsubsection{Formula for Unlikely Collisions}

To obtain lower bounds on testing juntas, it now suffices to design a collection $\cJ \subseteq
\binom{[n]}{k}$ of $k$-sets of variables, and a family $\cF$ of functions $\zo^k \to \zo$,
which satisfy the Unlikely Collisions condition of \cref{lemma:balls-and-bins}. We express this
condition in the following formula for $k$-junta setups.

\begin{boxproposition}
\label{prop:lb-unlikely-collision-formula}
For any $k$-junta setup $(\cJ, \cF)$, the Unlikely Collisions condition of
\cref{lemma:balls-and-bins} may be written as
\[
  \sum_{\Delta = 0}^k \Pru{\bm S, \bm T \sim \cJ}{|\bm S \cap \bm T| = \Delta}
    \Exu{\bm f, \bm g \sim \cF}
      { \Exu{\bm z \sim \zo^\Delta}{ 2\rho_{\bm f}(\bm S, \bm z)\rho_{\bm g}(\bm T, \bm z) }^m }
  = (1+o(1)) \frac{1}{N} .
\]
\end{boxproposition}
\begin{proof}
For any fixed $S, T, f, g$, with $|S \cap T| = \Delta$, using the independence of the $m$ samples
$\bm{x}_i \in \bm X$, the probability that the balls $f_S(\bm{\bm X})$ and $g_T(\bm S)$ collide
(\ie $f_S(\bm X) = g_T(\bm X)$) is
\begin{align*}
  \Pru{\bm X}{ f_S(\bm X) = g_T(\bm X) }
  &= \Exu{\bm z \sim \zo^\Delta}{ \rho_f(S, \bm z) \rho_g(T, \bm z)
      + (1 - \rho_f(S, \bm z))(1 - \rho_g(T, \bm z)) }^m \\
  &= \Exu{\bm z \sim \zo^\Delta}{1 - \rho_f(S, \bm z) - \rho_g(T, \bm z) + 2\rho_f(S, \bm z) \rho_g(T, \bm z)}^m \\
  &= \Exu{\bm z \sim \zo^\Delta}{2\rho_f(S, \bm z) \rho_g(T, \bm z)}^m 
\end{align*}
where in the last equality we used \Cref{eq:expectation of rho}.
We may now rewrite the goal as
\begin{equation}
\label{eq:nh_lb_0_distance_decomp}
  \sum_{\Delta = 0}^k \Pru{\bm S, \bm T \sim \cJ}{|\bm S \cap \bm T| = \Delta}
    \Exu{\bm f, \bm g \sim \cF}
      { \Exu{\bm z \sim \zo^\Delta}{ 2\rho_{\bm f}(\bm S, \bm z)\rho_{\bm g}(\bm T, \bm z) }^m }
  = (1+o(1)) \frac{1}{N} \,.
\end{equation}
\end{proof}

\subsection{Lower Bound Under the Uniform Distribution}

Let $\cJ = \binom{[n]}{k}$ be the collection of all $k$-subsets of $[n]$, and let $\cF$ be the set
of $\binom{2^k}{2^{k-1}}$ balanced functions. We will show that for $m = o\left(\sqrt{2^k
\log\binom{n}{k}}\right)$ the Unlikely Collisions condition holds
(\cref{lemma:lb-unlikely-collisions}).

For fixed $\Delta$, $S, T$ and $f, g$, define
\[
\bm{R} = \bm{R}(S,T,f,g) \define \Exu{\bm z}{2\rho_{f}(S,\bm z) \rho_{g}(T,
\bm z)} .
\]
To verify the Unlikely Collisions condition of \cref{lemma:balls-and-bins},
we need an expression for $\Ex{\bm{R}^m}$. We complete the proof in the following steps:
\begin{enumerate}
\item In \cref{section:lb-serious-calculus}, we obtain an expression for $\Ex{\bm{R}^m}$,
  assuming a concentration inequality for $\bm{R}$.
\item In \cref{section:lb-concentration}, we establish the appropriate concentration inequality.
\item In \cref{section:lb-calculation}, we complete the calculation to prove the Unlikely Collisions
condition.
\end{enumerate}

\subsubsection{Expression for $\Ex{\bm{R}^m}$ assuming concentration of $\bm{R}$}
\label{section:lb-serious-calculus}

\newcommand{\CE}{\Gamma}
\begin{boxproposition}
\label{prop:serious-calculus}
Assume that for every $k, \Delta$ and every $S, T$ satisfying $|S \cap T| = \Delta$, that the
concentration inequality
\[
  \forall \lambda \in (0,1) :\qquad
    \Pru{\bm f, \bm g}{ \Ex{\bm{R}} > \frac 1 2 + \lambda } \leq e^{-\CE(\Delta,k) \cdot \lambda^2} 
\]
holds for some function $\CE(\Delta,k)$.
Then for all $\Delta, k$, and $m$
\[
  \Exu{\bm f, \bm g}{ \bm{R}^m } \leq \frac{1}{N} \left( 1 + O\left( \frac{m}{\sqrt{ \CE(\Delta,k) }}
e^{\frac{m^2}{\CE(\Delta,k)}} \right) \right) \,.
\]
\end{boxproposition}
\begin{proof}
We write $\CE \define \CE(\Delta, k)$ for convenience.  Since $\bm{R}$ is a non-negative random
variable,
\begin{align*}
\Exu{\bm f, \bm g}{ \bm{R}^m }
&\leq \frac{1}{N} \Pr{ \bm{R} \leq \tfrac 1 2 }
    + \Pr{ \bm{R} > \tfrac 1 2 } \Exuc{}{\bm{R}^m }{\bm{R} > \tfrac 1 2 } \tag{$N=2^m$}\\
&= \frac{1}{N} \Pr{ \bm{R} \leq \tfrac 1 2 }
    + \Pr{ \bm{R} > \tfrac 1 2 } \int_0^\infty \Pruc{}{ \bm{R}^m \geq \gamma }{ \bm{R} > \tfrac 1 2 }
      d\gamma \\
&= \frac{1}{N} \Pr{ \bm{R} \leq \tfrac 1 2 }
    + \Pr{ \bm{R} > \tfrac 1 2 } \left( \frac{1}{N} + \int_{1/N}^\infty \Pruc{}{ \bm{R}^m \geq \gamma }{ \bm{R} > \tfrac 1 2 }
      d\gamma \right) \\
&= \frac{1}{N} + \Pr{ \bm{R} > \tfrac 1 2 } \int_{1/N}^\infty \Pruc{}{ \bm{R}^m \geq \gamma }{ \bm{R} > \tfrac 1 2 }
      d\gamma .
\end{align*}
Change the variables in the integral by defining $\lambda > 0$ such that
$\frac 1 2 + \lambda = \gamma^{1/m}$,
so
\[
  d\lambda = \frac{1}{m} \gamma^{(1-m)/m} d\gamma \qquad\equiv\qquad d\gamma = m \gamma^{(m-1)/m}
d\lambda = m \left(\frac 1 2 + \lambda\right)^{m-1} d\lambda .
\]
At $\gamma = 1/N$ we have $\lambda = 0$,
so the integral term becomes
\begin{align*}
&\Pr{ \bm{R} > \tfrac 1 2 } \int_{1/N}^\infty \Pruc{}{ \bm{R} \geq \gamma^{1/m} }{ \bm{R} > \tfrac 1 2 }
      d\gamma \\
&\qquad= \Pr{ \bm{R} > \tfrac 1 2 } \int_0^\infty \Pruc{}{ \bm{R} \geq \frac 1 2 + \lambda }{ \bm{R} > \tfrac 1 2 }
      m \left(\frac 1 2 + \lambda\right)^{m-1} d\lambda \\
&\qquad= \int_0^\infty \Pr{ \bm{R} \geq \frac 1 2 + \lambda } 
      m \left(\frac 1 2 + \lambda\right)^{m-1} d\lambda \\
&\qquad= \frac{1}{2^{m-1}} m \int_0^\infty \Pr{ \bm{R} \geq \frac 1 2 + \lambda } 
      \left(1 + 2\lambda\right)^{m-1} d\lambda \\
&\qquad\leq \frac{1}{N} \cdot 2 m \int_0^\infty e^{2\lambda(m-1) - \CE \lambda^2} d\lambda 
\end{align*}
where we have used the concentration assumption in the final line. Rewrite the exponent
in the integral as
\begin{align*}
2\lambda(m-1) - \CE \lambda^2
&= -\CE \left( \lambda^2 - \frac{2 \lambda (m-1)}{\CE} \right) \\
&= -\CE \left( \left(\lambda - \frac{(m-1)}{\CE}\right)^2 - \frac{(m-1)^2}{\CE^2} \right) \\
&= \frac{(m-1)^2}{\CE} - \CE \left(\lambda-\frac{(m-1)}{\CE}\right)^2 .
\end{align*}
Setting $t = \sqrt{\CE}\left(\lambda - \frac{(m-1)}{\CE}\right)$ so that $d\lambda = \frac{1}{\sqrt
\CE}
dt$, the integral becomes
\begin{align*}
  \int_0^\infty e^{2\lambda(m-1)-\CE\lambda^2} d\lambda
  = e^{\frac{(m-1)^2}{\CE}} \frac{1}{\sqrt{\CE}} \int_{-(m-1)}^\infty e^{-t^2} dt  
  \leq \frac{\sqrt{\pi}}{\sqrt \CE} e^{\frac{m^2}{\CE}} \,.
\end{align*}
Then
\begin{align*}
  \Ex{\bm{R}^m}
  \leq \frac{1}{N}\left( 1 + O\left( \frac{m}{\sqrt{\CE}} e^{\frac{m^2}{\CE}} \right) \right) \,,
\end{align*}
as desired.
\end{proof}

\subsubsection{Concentration of $\bm R$}
\label{section:lb-concentration}

In this section we prove the concentration of the variable $\bm R$ for fixed $\Delta = |S \cap T|$
and random functions $\bm f, \bm g$.  We have
\[
  \bm{R} = \bm{R}(S, T, \bm f, \bm g) = \Exu{\bm z \sim \zo^\Delta}{2\rho_{\bm f}(S, \bm z) \rho_{\bm g}(T, \bm z)} .
\]
For convenience, we define
\[
  K \define 2^k ,\qquad D \define 2^\Delta
\]
and
\[
\bm{F}_z \define \rho_{\bm f}(S, z) - \frac12 ,\qquad \bm{G}_z \define \rho_{\bm g}(T, z) - \frac12 .
\]
Recalling \cref{eq:lb-rho-sum}, we may write
\begin{equation}
\label{eq:lb-f-decomposition}
\bm{F}_z = \frac12 \cdot \frac{D}{K} \sum_{i=1}^{D/K} \bm{X}_{z,i}
\end{equation}
where the random variables $\{ \bm{X}_{z,i} \;|\; z \in \zo^\Delta, i \in [K/D] \}$ take
values in $\pmset$ and (since $\bm f$ is a uniformly random balanced function) are uniformly
distributed conditional on
\[
  0 = \sum_{z\in \zo^{\Delta}} \bm{F}_z = \frac12 \frac{D}{K} \sum_{z\in \zo^{\Delta}} \sum_{i\in [K/D]} \bm{X}_{z,i} .
\]
A similar statement holds for $\bm{G}_z$.  We may rewrite $\bm R$ as
\begin{align*}
  \bm{R}
  = \Exu{\bm z \sim \zo^\Delta}{2\rho_{\bm f}(S, \bm z) \rho_{\bm g}(T, \bm z)} 
  = \frac{2}{D} \sum_{z \in \zo^\Delta}
    \left(\frac14 + \frac12 \bm{F}_z + \frac12 \bm{G}_z + \bm{F}_z \bm{G}_z\right) 
  = \frac{1}{2} + \frac{2}{D} \inn{\bm F, \bm G} \,.
\end{align*}
To apply \cref{prop:serious-calculus}, we are now looking for an inequality of the form
\begin{equation}
\label{eq:lb-r-to-sum}
  \Pr{ \bm{R} > \frac{1}{2} + \lambda} \leq e^{-\Gamma(\Delta,k) \cdot \lambda^2}
  \;\equiv\; \Pr{ \inn{\bm F, \bm G} > D \cdot \lambda/2 } \leq e^{-\CE(\Delta,k) \cdot \lambda^2} \,.
\end{equation}
To obtain this inequality, we will use the properties of sub-gaussian, sub-exponential, and
negatively associated random variables.

\begin{definition}[Sub-Gaussian and Sub-Exponential]
A random variable $\bm Z$ is \emph{sub-gaussian with parameter}\footnote{If we define $\|\bm Z\|_{\psi_2}$ as the maximum parameter satisfying the desired inequality, then $\|\bm Z\|_{\psi_2}$ is the \emph{subgaussian norm} of $\bm Z$. Likewise, we can define the \emph{subexponential norm} of $\bm Z$.} 
$\|\bm Z\|_{\psi_2}$
when
\[
  \forall \lambda \geq 0 \,\colon\qquad \Pr{ |\bm Z| \geq \lambda } \leq 2 e^{-\lambda^2 / \|\bm
Z\|_{\psi_2}^2} .
\]
A random variable $\bm Z$ is \emph{sub-exponential with parameter} $\|\bm Z\|_{\psi_1}$
when
\[
  \forall \lambda \geq 0 \,\colon\qquad \Pr{ |\bm Z| \geq \lambda } \leq 2 e^{-\lambda / \|\bm
Z\|_{\psi_1}} .
\]
\end{definition}

\begin{definition}[Negative Associativity]
A sequence $\bm Z = (\bm Z_1, \dotsc, \bm Z_n) \in \bR^n$ of random variables are \emph{negatively associated} if
for every two functions $f, g \colon \bR^n \to \bR$ that depend on disjoint sets of variables
and are either both monotone increasing or both monotone decreasing, it holds that
\[
  \Ex{ f(\bm Z) g(\bm Z) } \leq \Ex{ f(\bm Z) } \Ex{ g(\bm Z) } .
\]
\end{definition}
We require the following convenient closure properties of negatively associated random variables
(see \eg \cite{Wajc17}).
\begin{proposition}[Closure properties]
Let $\bm Z = (\bm Z_1, \dotsc, \bm Z_n) \in \bR^n$ and $\bm W = (\bm W_1, \dotsc, \bm W_n) \in
\bR^n$ be independent sequences of random variables such that $\bm Z$ and $\bm W$ are each
negatively associated. Then
\begin{itemize}
\item The union $(\bm Z_1, \dotsc, \bm Z_n, \bm W_1, \dotsc, \bm W_n)$ is negatively associated.
\item For any sequence of functions $f_1, \dotsc, f_k \colon \bR^n \to \bR$ defined on pairwise
disjoint sets of variables, such that either all $f_i$ are monotone increasing or all $f_i$ are
monoton decreasing, the random variables
\[
  f_1(\bm Z), f_2(\bm Z), \dotsc, f_k(\bm Z)
\]
are negatively associated.
\end{itemize}
\end{proposition}
Negatively associated random variables satisfy similar concentration inequalities as independent
ones. We use the following form of a Chernoff-Hoeffding bound for negatively associated random
variables (see \eg \cite{Wajc17})

\begin{theorem}
\label{thm:lb-na-chernoff}
Let $\bm{Z}_1, \dotsc, \bm{Z}_n$ be negatively associated, mean 0 random variables taking values in
$[-a,a]$. Then
\[
  \Pr{ \left|\sum_{i=1}^n \bm{Z}_i \right| > \lambda } \leq 2 \cdot \exp\left(- \frac{\lambda^2}{2na^2}\right)
.
\]
\end{theorem}

The following theorem is essentially identical to Theorem 2.8.1 of \cite{Ver18} except that
it allows negatively-associated variables instead of independent ones. It follows from the same
proof as in \cite{Ver18}, using the properties of negative associativity: 

\begin{theorem}
\label{thm:lb-na-exp}
There is a universal constant $c > 0$ such that the following holds.
Let $\bm{Z}_1, \dotsc, \bm{Z}_n$ be negatively associated sub-exponential 0-mean random variables. Then
\[
  \Pr{ \sum_{i=1}^n \bm{Z}_i > \lambda } \leq
    \exp\left( - c\cdot \min\left\{ \frac{\lambda^2}{\sum_{i=1}^n \|\bm{Z}_i\|_{\psi_1}},
                            \frac{\lambda}{\max_i \|\bm{Z}_i\|_{\psi_1}} \right\}
\right) .
\]
\end{theorem}

\begin{boxproposition}[Properties of the variables $\bm{F}_z, \bm{G}_z$]
\label{prop:lb-properties-FG}
There exist universal constant $c_1, c_2 > 0$ such that
the random variables $\bm{F}_z, \bm{G}_z$ satisfy:
\begin{enumerate}
\item Each variable $\bm{F}_z$ and $\bm{G}_z$ is sub-gaussian with parameters
$\|\bm{F}_z\|_{\psi_2}, \|\bm{G}_z\|_{\psi_2} \leq c_1 \cdot D/K$;
\item Each variable $\bm{F}_z \bm{G}_z$ is sub-exponential with parameter $\|\bm{F}_z
\bm{G}_z\|_{\psi_1} \leq c_2 \cdot D^2/K^2$;
\item The variables $\{\bm{F}_z \bm{G}_z\}_{z \in \{0, 1\}^\Delta}$ are negatively associated.
\end{enumerate}
\end{boxproposition}
\begin{proof}
For each $z$, writing $\bm{F}_z = \frac{1}{2} \frac{D}{K} \sum_{i=1}^{K/D} \bm{X}_{z,i}$ where
$\bm{X}_{z,i} \in \{\pm 1\}$ are random variables with mean 0. and the collection of random
variables $\{ \bm{X}_{z,i} | z \in \zo^\Delta, i \in [K/D]\}$ are uniformly distributed under the
condition $\sum_z \sum_i \bm{X}_{z,i} = 0$.  Then the random variables $\{\bm{X}_{z,i}\}_{z,i}$ are
negatively associated (this is a standard example of negatively associated random variables, see
Theorem 10 of \cite{Wajc17}). Due to the closure properties (in this case, taking a subset of variables), 
for each $z \in \zo^\Delta$, $\bm{F}_z$ is a sum of negatively associated random variables.
Therefore, by the Chernoff-Hoeffding bound for negatively associated random variables
(\cref{thm:lb-na-chernoff}), there is
some constant $c_1' > 0$ such that
\[
  \forall \lambda > 0 : \qquad \Pr{ |\bm{F}_z| > \lambda } \leq 2e^{- c_1' \frac{K}{D} \lambda^2 } \,.
\]
The same holds for $\bm{G}_z$, so this proves that these variables satisfy the required sub-gaussian
properties. Then the required sub-exponential property on $\bm{F}_z \bm{G}_z$ holds due to
well-known facts about products of sub-gaussian random variables (see \eg Lemma 2.7.7 of
\cite{Ver18}).

It remains to prove that the variables $\{ \bm{F}_z\bm{G}_z \;|\; z \in \zo^\Delta \}$ are
negatively associated. This again follows from the closure properties, since the union of variables
$\{ \bm{X}_{z,i} \}_{i,z}$ and their counterparts for the variables $\bm{G}_z$ are negatively
associated, and for each $z$ the value $\bm{F}_z\bm{G}_z$ is a monotone increasing function on a
subset of these variables, with the respective subsets of variables for each $z$ being disjoint.
\end{proof}

Applying the concentration inequality for sums of negatively associated sub-exponential random
variables (\cref{thm:lb-na-exp}) to the sum $\inn{\bm F, \bm G} = \sum_z \bm{F}_z \bm{G}_z$ over the
$D$ variables $\bm{F}_z \bm{G}_z$, using the sub-exponential parameters from
\cref{prop:lb-properties-FG}, we obtain the desired concentration inequality:

\begin{boxlemma}
\label{lemma:lb-concentration}
There exists a universal constant $c > 0$ such that the following holds:
\[
  \forall \lambda \in (0,1) :\qquad \Pr{ \inn{\bm F, \bm G} > D \cdot \lambda/2 }
    \leq \exp\left( - c \cdot \frac{K^2}{D} \lambda^2 \right)  \,.
\]
As a consequence of \cref{eq:lb-r-to-sum} the same upper bound holds on $\Pr{ \bm{R} > \frac12 +
\lambda }$.
\end{boxlemma}

\subsubsection{Proof of Unlikely Collisions}
\label{section:lb-calculation}

We finally establish the Unlikely Collisions condition. By
\cref{lemma:lb-k-junta,lemma:lb-k-feature}, this establishes \cref{thm:lb-testing-unif}.

\begin{boxlemma}
\label{lemma:lb-unlikely-collisions}
Let $\cJ = \binom{[n]}{k}$ be the set of all $k$-sets and let $\cF$ be the set of all balanced
functions $\zo^k \to \zo$. Assume $\log\binom{n}{k} < \beta 2^k$ for some constant $\beta > 0$, and
$k \leq n/e$.  Then for $m = o( \sqrt{2^k \log\binom{n}{k} } )$, the $k$-junta setup $(\cJ, \cF)$
satisfies the Unlikely Collisions condition of \cref{lemma:balls-and-bins}; in other words, 
\[
  \Exu{\bm f, \bm g, \bm S, \bm T}{ \Pru{\bm X}{ \bm{f_S}(\bm X) = \bm{g_T}(\bm X) } }
    = \frac{1}{N} (1 + o(1)).
\]
\end{boxlemma}
\begin{proof}
By assumption, for every constant $\alpha > 0$, we have $m < \alpha \sqrt{2^k \log\binom{n}{k}}$ 
for sufficiently large $n,k$.
By \cref{prop:lb-unlikely-collision-formula}, the definition of $\bm R$, and
the combination of \cref{lemma:lb-concentration} and \cref{prop:serious-calculus},
we have
\begin{align*}
  \Exu{\bm f, \bm g, \bm S, \bm T}{ \Pru{\bm X}{ \bm{f_S}(\bm X) = \bm{g_T}(\bm X) } }
  &= \sum_{\Delta=0}^k \Pr{ |\bm S \cap \bm T| = \Delta } \cdot \Ex{ \bm{R}^m } \\
  &\leq \frac{1}{N} \sum_{\Delta=0}^k \Pr{ |\bm S \cap \bm T| = \Delta } \cdot 
    \left( 1 + O\left( \frac{m}{\sqrt{2^{2k-\Delta}}} \cdot \exp\left(\frac{m^2}{2^{2k-\Delta}}\right) \right)
\right) \\
  &= \frac{1}{N}\left( 1 + O\left( \sum_{\Delta=0}^k \Pr{ |\bm S \cap \bm T| = \Delta }
      \cdot \frac{m}{\sqrt{2^{2k-\Delta}}} \cdot \exp\left(\frac{m^2}{2^{2k-\Delta}}\right) \right) \right) \,,
\end{align*}
so we want to show that
\begin{equation}
\label{eq:lb-goal-sum-vanishes}
  \sum_{\Delta=0}^k \Pr{ |\bm S \cap \bm T| = \Delta }
      \cdot \frac{m}{\sqrt{2^{2k-\Delta}}} \cdot e^{\frac{m^2}{2^{2k-\Delta}}}
  = o(1) .
\end{equation}
It can be easily checked that
\[
  \Pr{|\bm S \cap \bm T| = \Delta} \leq \left(\frac{k}{n}\right)^\Delta ,
\]
so our sum becomes
\begin{align*}
  \sum_{\Delta=0}^k \Pr{ |\bm S \cap \bm T| = \Delta }
      \cdot \frac{m}{\sqrt{2^{2k-\Delta}}} \cdot e^{\frac{m^2}{2^{2k-\Delta}}}
  &\leq \alpha \sum_{\Delta=0}^k \frac{k^\Delta}{n^\Delta} \cdot \sqrt{\frac{\ln
\binom{n}{k}}{2^{k-\Delta}}} \cdot e^{\alpha^2 \cdot \frac{1}{2^{k-\Delta}} \ln\binom{n}{k}} \\
\end{align*}
Define
\[
  T(\Delta) \define \sqrt{2^\Delta} \left(\frac{k}{n}\right)^\Delta \binom{n}{k}^{\alpha^2 \cdot
2^{\Delta-k}} ,
\]
so that our sum is
\[
  \alpha \sqrt{2^{-k} \ln\binom{n}{k}} \sum_{\Delta=0}^k T(\Delta)
  \leq \alpha \beta \sum_{\Delta=0}^k T(\Delta) .
\]
Observe that for $\Delta \geq 1$,
\begin{equation}
\label{eq:lb-t-monotone}
\frac{T(\Delta)}{T(\Delta-1)} = \sqrt 2 \left(\frac{k}{n}\right) \binom{n}{k}^{\alpha^2
2^{\Delta-1-k}} ,
\end{equation}
so that this fraction is monotone increasing with $\Delta$. We split the sum into three parts:
$\Delta \leq s$, $s < \Delta < t$, and $t \leq \Delta$, where $s$ and $t$ are chosen such that
\[
  s < \Delta < t \iff \frac{T(\Delta)}{T(\Delta-1)} \in [ 1-\delta, 1+\delta ] ,
\]
for some constant $\delta > 0$. Therefore, by monotonicity of the ratio \eqref{eq:lb-t-monotone},
\begin{align*}
  \sum_{\Delta=0}^k T(\Delta)
  &= \sum_{\Delta=0}^s T(\Delta) + \sum_{\Delta= s+1}^{t-1} T(\Delta) + \sum_{\Delta=t}^k T(\Delta) \\
  &\leq \sum_{\Delta=0}^s T(0) (1-\delta)^\Delta + \sum_{\Delta=t}^k T(k) (1+\delta)^{\Delta-k} + \sum_{\Delta= s+1}^{t-1} T(\Delta)  \\
  &\leq O\big( T(0) + T(k) + (t-s)\max(T(s), T(t)) \big) .
\end{align*}
To bound $t-s$, note that for all $s < \Delta < t$,
\[
  (1-\delta) \frac{1}{\sqrt 2} \left(\frac{n}{k}\right)
  < \binom{n}{k}^{\alpha^2 2^{\Delta-1-k}} < (1+\delta) \frac{1}{\sqrt 2} \left(\frac{n}{k}\right) .
\]
For the left inequality, we require
\begin{align*}
  &(1-\delta) \frac{1}{\sqrt 2} \left(\frac{n}{k}\right)
    < \left(\frac{e n}{k}\right)^{\frac{k}{2^k} \alpha^2 2^{\Delta-1}}
  &\equiv \log\left(\frac{n}{k}\right) - \log\left(\frac{\sqrt 2}{1-\delta}\right)
    < \frac{k}{2^k} \alpha^2 2^{\Delta-1} \left(\log\left(\frac{n}{k}\right) + \log(e)\right) ,
\end{align*}
so in particular $\Delta > k - \log(k) + C$ for some constant $C$.
For the right inequality, we require
\begin{align*}
  \left(\frac{n}{k}\right)^{\frac{k}{2^k} \alpha^2 2^{\Delta-1}}
  < (1+\delta) \frac{1}{\sqrt 2} \left(\frac{n}{k}\right)
  &\equiv \frac{k}{2^k} \alpha^2 2^{\Delta-1} \log\left(\frac{n}{k}\right)
  < \log\left(\frac{n}{k}\right) + \log\left(\frac{1+\delta}{\sqrt 2}\right) ,
\end{align*}
so in particular we require $\Delta < k - \log(k) + C'$ for some constant $C'$. Therefore $t-s \leq
C'-C = O(1)$. So what remains is to bound $O(T(0) + T(k))$. By the assumption $\log\binom{n}{k} \leq
\beta 2^k$,
\begin{align*}
  T(0) = \binom{n}{k}^{\alpha^2 2^{-k}} \leq 2^{\alpha^2 \beta } .
\end{align*}
By the assumption $\log(n/k) \geq \log(e)$,
\begin{align*}
  T(k) = \sqrt{2^k} \left(\frac{k}{n}\right)^k \binom{n}{k}^{\alpha^2}
  \leq \sqrt{2^k} \left(\frac{k}{n}\right)^k \left(\frac{en}{k}\right)^{\alpha^2 k }
  = 2^{k\left(\tfrac{1}{2} + \alpha^2\log(e) - (1-\alpha^2)\log(n/k)\right)}  < 1 .
\end{align*}
We may now conclude that \cref{eq:lb-goal-sum-vanishes} is satisfied, since
\[
  \alpha \beta \sum_{\Delta=0}^k T(\Delta) = \alpha \cdot O(T(0)+T(k) + (t-s) \max(T(0), T(k))) 
= O(\alpha) . \qedhere
\]
\end{proof}
\ignore{

For $\gamma > \log (\ln\binom{k}{n})$ the sum is geometric and converges to $O(\alpha)$, so
we must only bound the sum for $\gamma \leq \log\ln\binom{n}{k}$. In this case,
\[
  \gamma \leq \log\ln\binom{n}{k} < \log(\beta 2^k) = k - \log(1/\beta) \,.
\]
Since $\binom{n}{k} \leq \left(\frac{e n}{k}\right)^k$, we have
$\ln\binom{n}{k} \leq k \ln(en/k)$. Then the sum is at most
\[
  \sum_{\gamma=0}^{\log\ln\binom{n}{k}}
    \exp\left(\alpha\frac{1}{2^\gamma} k \ln(en/k) + \tfrac12 \ln\ln\binom{n}{k} - \tfrac12
\ln(2) \gamma - (k-\gamma)\ln(n/k)\right) .
\]
The exponent is
\begin{align*}
&\left(\frac{\alpha}{2^\gamma}k + \gamma - k\right) \ln(n/k) + \frac{\alpha}{2^\gamma} k +
\tfrac12 \ln\ln\binom{n}{k}  - \tfrac12 \ln(2) \gamma \,,
\end{align*}
We split into two more cases, $\gamma \leq \log k$ and $\gamma > \log k$. For $\gamma \leq \log k$,
using the assumption that $\ln(n/k) \geq 1$ (\ie $n > ek$),
the exponent is at most
\begin{align*}
  &\left(\alpha k - k + \log k \right) \ln(n/k) + \alpha k + \tfrac12 \ln\ln\binom{n}{k} - \tfrac12
\ln(2) \gamma \\
  &\qquad\leq 
  \left(2\alpha k - k + \log k \right) \ln(n/k) + \tfrac12 \ln(2^k) - \tfrac12
\ln(2) \gamma \\
  &\qquad\leq 
  \left(2\alpha k - k/2 + \log k \right) \ln(n/k) - \tfrac12
\ln(2) \gamma \\
  &\qquad\leq 
  -\beta \ln(n/k) - \tfrac12 \ln(2)\gamma \,,
\end{align*}
where $\beta > 0$ is some constant, and we have assumed $\alpha < 1/5$. Therefore the sum for
$\gamma \leq \log k$ converges to $O\left((k/n)^\beta\right)$ for some constant $\beta$.

Now consider $\log k < \gamma < \log\ln\binom{n}{k}$. Write $\gamma = \log\ln\binom{n}{k} - t$.
Since $\ln\binom{n}{k} \leq \beta 2^k$ for some constant $\beta < 0$, we have
$\gamma \leq \log(\beta 2^k) - t = k - b - t$ where we write $b = \log(1/\beta)$.
Then the exponent is at most
\begin{align*}
  &( \alpha - k + \gamma )\ln(n/k) + \alpha + \tfrac12 \left(\ln\ln\binom{n}{k} - \ln\ln\binom{n}{k} +
t\ln(2)\right) \\
  &\qquad\leq (2\alpha - b - t)\ln(n/k) + \tfrac12\ln(2) t \\
  &\qquad\leq (2\alpha - b - t/2)\ln(n/k) .
\end{align*}
The sum of this exponential over values $t$ converges to at most $O((k/n)^{2\alpha-b})$, which
completes the proof.
}

\subsection{Dependence on $\epsilon$ for Product Distributions}

The above argument suffices to get a lower bound of $\Omega\left(\sqrt{2^k \log \binom{n}{k}} +
\log\binom{n}{k}\right)$ for any sufficiently small constant $\epsilon > 0$, even when the
underlying distribution is known to be uniform. Now we will show how to obtain a lower bound of
\begin{equation}
\Omega\left( \frac{1}{\epsilon} \left( \sqrt{2^k \log\binom{n}{k}} + \log\binom{n}{k} \right) \right)
\end{equation}
for a fixed product distribution known to the algorithm, completing the proof of
\cref{thm:intro-main}. This requires different arguments for testing and for feature selection.

We remark that the multiplicative dependence of $1/\epsilon$ is not possible for $\epsilon < 2^{-k}$
when the underlying distribution is uniform; see \cref{section:lb-eps-unif} for a proof of this.

\subsubsection{Lower Bound for Testing}

\begin{boxlemma}
\label{lemma:lb-eps-testing}
For any constant $\tau \in (0,1)$, and any $n$, there exists a product distribution $\mu$ over $\zo^n$
such that, for all $k \leq (1-\tau)n$ and $\epsilon > 2^{-\tau n}$, any $\epsilon$-tester
for $k$-juntas over $\mu$ requires sample size at least
\[
\Omega\left( \frac{1}{\epsilon} \left(\sqrt{ 2^k \log\binom{n}{k} } + \log\binom{n}{k} \right) \right) .
\]
\end{boxlemma}

For any $n$, $k$, and $q = q(n)$, we define a product distribution $\mu_q$ over $\zo^{n+q}$.
For convenience, we write $L = [n]$ for the first $n$ bits and $R = [n+q]\setminus [n]$ for the last
$q$ bits.  Each string $\bm{x} \sim \mu_q$ is chosen as follows:
\begin{itemize}
    \item For $i \in [q]$, we draw $\bm{x}_{n+i} \sim \Ber(2^{-i})$;
    \item For $i \in [n]$, we draw $\bm{x}_i \sim \Ber(1/2)$.
\end{itemize}
In other words, the distribution of $\bm{x}$ is a product of Bernoullis, which are uniformly random
in the first $n$ bits, and with exponentially decreasing parameters in the last $q$ bits.

Now we define a way to transform functions over $\zo^n$ into functions over $\zo^{n+q}$. Let $n+i^*
\in R$ be the coordinate such that $\epsilon/2 \leq 2^{-i^*} < \epsilon$ and let $f' \colon \zo^n
\to \zo$ be any function. Then we define $f \colon \zo^{n+q} \to \zo$ as
\begin{equation}
\label{eq:lb-eps-reduction}
  \bm f(x) \define \begin{cases}
      \bm f'(x_L) &\text{ if } x_{n+i^*} = 1 \\
      0 &\text{ if } x_{n+i^*} = 0.
  \end{cases}
\end{equation}
If $f'$ is a $k$-junta, then $f$ is a $(k+1)$-junta since it depends only on bit $n+i^*$ and
the $k$ bits from the prefix $L$. On the other hand, if $f'$ is $\delta$-far from being a $k$-junta over
the uniform distribution, then $f$ is at least $(\epsilon\delta/2)$-far from being a $(k+1)$-junta
over $\mu_q$:

\begin{boxproposition}
\label{prop:lb-reduction-distance}
For any $n,q$ and $1/2 \geq \epsilon > 2^{-q}$, suppose that $g' \colon \zo^n \to \zo$ is $\delta$-far from
being a $k$-junta over the uniform distribution. Let $i^* \define \lceil \log(1/\epsilon) \rceil$
and define $g \colon \zo^{n+q} \to \zo$ as
\[
  g(x) \define \begin{cases}
    g'(x_L) &\text{ if } x_{i^*} = 1 \\
    0       &\text{ if } x_{i^*} = 0 .
  \end{cases}
\]
Then $g$ is $(\epsilon \delta/2)$-far from being a $(k+1)$-junta over $\mu_q$.
\end{boxproposition}
\begin{proof}
Let $f \colon \zo^{n+q} \to \zo$ be any $(k+1)$-junta and write $S \subseteq [n+q]$ for its set of
relevant variables. First assume $S \cap R \neq \emptyset$ so that $|S \cap L| \leq k$. Since $g'$
is $\delta$-far from being a $k$-junta, we have for any $x_R \in \zo^q$,
\[
  \Pru{\bm{x}_L \sim \unif(\zo^n)}{ f(\bm{x}_L, x_R) = g'(\bm{x}_L) } > \delta .
\]
Now let $\bm{x} \sim \mu_q$ so that $\bm{x}_L$ is uniformly and independently distributed. Note that
$\Pr{x_{n+i^*} = 1} \geq \epsilon/2$. Then
\begin{align*}
\Pru{\bm x}{ f(\bm x) \neq g(\bm x) } &\geq 
  \Pr{\bm{x}_{n+i^*} = 1} \cdot \Exuc{\bm x}{ f(\bm{x}_L, \bm{x}_R) = g'(\bm{x}_L) }{\bm{x}_{n+i^*} = 1} 
\geq \frac{\epsilon}{2} \cdot \delta .
\end{align*}
Now suppose $S \cap R = \emptyset$ so that $f(x_L,x_R) = f'(x_L)$ for some $f'$. Fix any $x_L \in \zo^n$.
Then
\begin{align*}
  \Pru{\bm{x}_R}{ f(x_L, \bm{x}_R) \neq g(x_L, \bm{x}_R) }
  &= \Pr{\bm{x}_{n+i^*} = 0} \cdot \ind{f'(x_L) \neq 0}
    + \Pr{\bm{x}_{n+i^*} = 1} \cdot \ind{f'(x_L) \neq g'(x_R)} .
\end{align*}
If $g'(x_L) = 0$ then this is 1 when $f'(x_L) = 1$ and 0 otherwise. If $g'(x_L) = 1$ then this is at
least $\epsilon/2$ regardless of $f'(x_L)$. Since $g'$ is $\delta$-far from any $k$-junta, it must
take value 1 with probability at least $\delta$ over $\bm{x}_L$. Therefore
\begin{align*}
\dist_{\mu_q}(f,g) \geq \Pru{\bm{x}_L}{ g'(\bm{x}_L) = 1 } \cdot \epsilon/2 \geq \delta\epsilon/2 ,
\end{align*}
as desired.
\end{proof}

\begin{boxproposition}
\label{prop:lb-eps-testing-reduction}
For any $n, k, q \in \bN$, such that $k < n$, any $\delta \in (0,1)$, and $\epsilon > 2^{-q}$,
suppose that there exists a $(k+1)$-junta tester with distance parameter $(\epsilon\delta/2)$ under
distribution $\mu_q$ on $\zo^{n+q}$ which draws $m(n,k,\epsilon)$ samples. Then there is a $k$-junta
tester with distance parameter $\delta$ under $\unif(\zo^n)$ which draws $O( \epsilon \cdot
m(n,k,\epsilon\delta/2) )$ samples.
\end{boxproposition}
\begin{proof}
We design a tester for the uniform distribution over $\zo^n$ as follows. Let $f' \colon \zo^n \to
\zo$ be the input function and let $f \colon \zo^{n+q} \to \zo$ be the corresponding function
defined in \cref{eq:lb-eps-reduction}. Observe that, given access to uniform samples from $\zo^n$ labelled by
$f'$, we may simulate a sample $(\bm x, f(\bm x))$ with $\bm x \sim \mu_q$ as follows:
\begin{enumerate}
\item For each $i \in [q]$, sample $\bm{x}_{n+i} \sim \Ber(2^{-i})$. 
\item If $\bm{x}_{n+i^*} = 1$, sample $\bm{x}_L \sim \unif(\zo^n)$ and return
    $(\bm x, \bm f'(x_L))$; \label{item:sample and compute f}
\item Otherwise, if $\bm{x}_{n+i^*} = 1$, sample $\bm{x}_L \sim \unif(\zo^n)$ and return $(\bm x, 0)$.
\end{enumerate}
If $f'$ is a $k$-junta then $f$ is a $(k+1)$-junta, so the tester should accept with the correct
probability. If $f'$ is $\delta$-far from being a $k$-junta then by
\cref{prop:lb-reduction-distance} $f$ is $\epsilon\delta/2$-far from being a $(k+1)$-junta, so the
tester will reject with the correct probability.

Finally, \Cref{item:sample and compute f} is executed with probability $2^{-i^*} = O(\varepsilon)$. Thus, by standard concentration bounds, \Cref{item:sample and compute f} is executed at most $O(\varepsilon) \cdot m (n, k, \varepsilon\delta / 2)$ times in total.
\end{proof}

To complete the proof of \cref{lemma:lb-eps-testing}, let $n \in \bN$, let $\tau \in (0,1)$ be any
constant, let $k < (1-\tau)n$, let $\epsilon > 2^{-\tau n}$, and let $\delta > 0$ be a sufficiently
small constant. Set $q = \tau n$ so that $n = n' + q$ for $n' = (1-\tau)n$. Then by applying
\cref{prop:lb-eps-testing-reduction}, we obtain a $(k-1)$-junta tester with distance parameter
$\delta/2$ for the uniform distribution on $n' = (1-\tau)n$ bits, so that our lower bound from
\cref{thm:lb-testing-unif} applies.

\subsubsection{Lower Bound for Feature Selection}

\begin{boxlemma}
\label{lemma:lb-eps-testing}
For any constant $\tau \in (0,1)$, and any $n$, there exists a product distribution over $\zo^n$
such that, for all $k \leq (1-\tau)n$ and $\epsilon > 2^{-\tau n}$, any 
$k$-feature selector with parameter $\varepsilon$
requires sample size at least
\[
\Omega\left( \frac{1}{\epsilon} \left(\sqrt{ 2^k \log\binom{n}{k} } + \log\binom{n}{k} \right) \right) .
\]
\end{boxlemma}

We will use a similar reduction as for testing $k$-juntas.

\begin{boxproposition}
\label{prop:lb-eps-feature-reduction}
For any $n, k, q \in \bN$, such that $k < n$, any $\delta \in (0,1)$, and $\epsilon > 2^{-q}$,
suppose that there exists a $(k+1)$-feature selector with distance parameter $(\epsilon\delta/2)$ under
distribution $\mu_q$ on $\zo^{n+q}$ which draws $m(n,k,\epsilon)$ samples. Then there is a
$k$-feature selector with distance parameter $\delta$ under $\unif(\zo^n)$ which draws $O( \epsilon
\cdot m(n,k,\epsilon\delta/2) )$ samples.
\end{boxproposition}
\begin{proof}
We follow the strategy of the reduction for testing. Our goal is to design a $k$-feature selector
for the uniform distribution over $\zo^n$, with distance parameter $\delta$, by reduction to a
$(k+1)$-feature selector for $\mu_q$ over $\zo^{n+q}$ with distance parameter $\epsilon\delta/2$.
Given access to uniform samples from $\zo^n$ labelled by $f'$, we simulate samples $(\bm x, f(\bm
x))$ with $\bm x \sim \mu_q$ as in \cref{prop:lb-eps-testing-reduction}, where $f$ and $f'$ are
again defined as in \cref{eq:lb-eps-reduction}, with $i^* \in [q]$ being the coordinate such that
$\epsilon/2 \leq 2^{-i^*} < \epsilon$. We send the simulated samples of $f$ to the $(k+1)$-feature
selector for $\mu_q$, which produces a set $S \subseteq [n+q]$ of $|S| = k+1$ variables. We then
output the set $S \cap [n]$ unless $|S \cap [n]| = k+1$, in which case we output $\emptyset$.

Assume that the $(k+1)$-feature selector for $\mu_q$ succeeds, so that $f$ is
$(\epsilon\delta/2)$-close to some $(k+1)$-junta $g \colon \zo^{n+q} \to \zo$ on variables $S$.  Our
goal is to show that $S \setminus [q]$ has $|S \cap L| \leq k$ and that $f'$ is $\delta$-close to a
$k$-junta on variables $S \setminus [q]$.

First suppose that $|S \cap L| = k+1$, so that our algorithm outputs $\emptyset$ as the set
of relevant variables for the input $f'$. Let $g$ be the $(k+1)$-junta on variables $S$ minimizing
distance to $f$ over $\mu_q$. In this case $g$ does not depend on $i^* \in [q]$.  Therefore, for any
setting $z \in \zo^{n+q}$,
\[
  \epsilon\delta/2 \geq \Pruc{\bm x \sim \mu_q}{ g(\bm x) \neq f'(\bm x) }{ \bm{x}_S = z_S }
  = \Pruc{\bm x \sim \mu_q}{ g(z) \neq f(\bm x) }{ \bm{x}_S = z_S } .
\]
If $g(z) = 1$ this leads to a contradiction since $f'(\bm x) = 0$ when $\bm{x}_{i^*} = 0$ which
occurs with probability $\geq \epsilon/2 > \epsilon\delta/2$. So it must be the case that $g$ is
the constant 0 function. Now
\begin{align*}
  \Pru{\bm{x}_L \sim \unif(\zo^n)}{ f'(\bm x) = 1 }
  &= \Pruc{\bm{x} \sim \mu_q}{ f(\bm x) = 1 }{ \bm{x}_{i^*} = 1 } \\
  &\leq \frac{2}{\epsilon} \Pru{\bm{x} \sim \mu_q}{ f(\bm x) = 1 }
  = \frac{2}{\epsilon} \Pru{\bm{x} \sim \mu_q}{ f(\bm x) \neq g(\bm x) }
  \leq \frac{2}{\epsilon} \frac{\epsilon\delta}{2} = \delta ,
\end{align*}
so the input $f'$ is $\delta$-close to constant and the algorithm succeeds.

Next suppose that $|S \cap L| \leq k$ so that the algorithm outputs $S \cap L$. Let $g$ be the
$(k+1)$-junta on variables $S$ minimizing distance to $f$ over $\mu_q$. Then
\begin{align*}
\frac{\epsilon\delta}{2}
&\geq \Exu{\bm{x}_R }{ \Pru{\bm{x}_L}{ g(\bm{x}_L, \bm{x}_R) \neq f(\bm{x}_L, \bm{x}_R) }} \\
&= \Pr{\bm{x}_{i^*} = 0} \cdot
      \Exuc{\bm{x}_R }{ \Pru{\bm{x}_L}{ g(\bm{x}_L, \bm{x}_R) \neq 0 }}{ \bm{x}_{i^*} = 0 } \\
&\qquad+ \Pr{\bm{x}_{i^*} = 1} \cdot
      \Exuc{\bm{x}_R }{ \Pru{\bm{x}_L}{ g(\bm{x}_L, \bm{x}_R) \neq f(\bm{x}_L, \bm{x}_R) }}{ \bm{x}_{i^*} = 1 } \\
&\geq \frac{\epsilon}{2} \cdot
  \Exuc{\bm{x}_R }{ \Pru{\bm{x}_L}{ g(\bm{x}_L, \bm{x}_R) \neq f(\bm{x}_L, \bm{x}_R) }}{ \bm{x}_{i^*} = 1 } .
\end{align*}
Then there exists a fixed assignment $\bm{x}_R = z$ such that $\Pru{\bm{x}_L \sim \unif(\zo^n)}{
g(\bm{x}_L, z) = f(\bm{x}_L, z) } \leq \delta$. For fixed $z$, $g(\cdot, z)$ depends only on the
variables $S \cap L$ while $f(\cdot, z) = f'(\cdot)$, so this proves correctness of the output $S
\cap L$.

Similarly to the proof of \Cref{prop:lb-eps-testing-reduction}, we can bound the total number of samples needed using standard concentration bounds.
\end{proof}

To complete the proof of \cref{lemma:lb-eps-testing}, we combine the reduction in
\cref{prop:lb-eps-feature-reduction} with our lower bound in \cref{thm:lb-testing-unif}, using the
same calculations as in the lower bound for testing to get the condition $k < (1-\tau)n/2$.

%% file: appendix_lower.tex
\section{Lower Bound on Testing Junta Truncation}
\label{section:junta-truncation}

Given parameters $k,n \in \bN$ and $\epsilon > 0$, an algorithm \emph{tests $k$-junta truncation}
if for every probability distribution $\cD$ over $\zo^n$ it draws $m = m(n,k,\epsilon)$ samples from
$\cD$, and its output satisfies:
\begin{enumerate}
    \item If there exists a $k$-junta $f \colon \zo^n \to \zo$ such that $\cD$ is the uniform
        distribution over $f^{-1}(1)$, output $\ACCEPT$ with probability at least $3/4$.
    \item If $\cD$ is the uniform distribution, output $\REJECT$ with probability at least $3/4$.
\end{enumerate}

We may now prove a lower bound of
    \[
        \Omega\left(\sqrt{2^k \log\binom{n}{k}} + \log\binom{n}{k} \right)
    \]
for testing junta trunction.

\begin{proof}[Proof of \cref{thm:intro-junta-truncation}]
    Fix a $k$-junta setup $(\cJ, \cF)$ on domain $\zo^n$.  Consider the following task. We are given
    $m$ uniformly samples $\bm X \sim \zo^n$ together with a random sequence of labels $\bm \ell \in \zo^m$
    generated either as:
    \begin{enumerate}
        \item $\bm \ell \sim \zo^m$ uniformly at random. We write this distribution over $(\zo^n)^m
            \times \zo^m$ as $\cD_\unif$.
        \item $\bm \ell = \bm{f}_{\bm S}(\bm X)$ where $\bm f \sim \cF$ and $\bm S \sim \cJ$. We
            write this distribution over $(\zo^n)^m \times \zo^m$ as $\cD_J$.
    \end{enumerate}
    Our task is to distinguish which of these cases we are in.  From our proof of
    \cref{thm:lb-testing-unif}, using \cref{prop:lb-junta-setups}, we have the following statement.

    \begin{claim}
        If $(\cJ, \cF)$ and $m$ are chosen to satisfy \cref{prop:lb-junta-setups} with distance
        parameter $\epsilon < 1/100$, then the TV distance between the distributions $\cD_\unif$ and $\cD_J$
        is at most $2/100$.
    \end{claim}
    \begin{proof}[Proof of claim]
      Consider any event $E \in (\zo^n)^m \times \zo^m$. We have
      \begin{align*}
        &\left|\Pru{(\bm X, \bm \ell) \sim \cD_\unif}{ (\bm X, \bm \ell) \in E }
            - \Pru{(\bm X, \bm \ell) \sim \cD_J}{ (\bm X, \bm \ell) \in E } \right| \\
        &\qquad= \left|\Exu{\bm X}{ \Pru{\bm \ell \sim \unif(\zo^m)}{  (\bm X, \bm \ell) \in E }
            - \Pru{\bm{f}_{\bm S} \sim \cD_J}{ (\bm X, \bm{f}_{\bm S}(\bm X)) \in E } } \right| \\
        &\qquad\leq \Exu{\bm X}{ \left|\Pru{\bm \ell \sim \unif(\zo^m)}{  (\bm X, \bm \ell) \in E } 
            - \Pru{\bm{f}_{\bm S} \sim \cD_J}{ (\bm X, \bm{f}_{\bm S}(\bm X)) \in E } \right| }  \\
        &\leq \Exu{\bm X}{ \| \unif(\zo^m) - \bm{f}_{\bm S}(\bm X) \|_\TV } 
        \leq \frac{1}{100} + \epsilon \,,
      \end{align*}
      where the final inequality is due to \cref{prop:lb-junta-setups}.
    \end{proof}

    Now we reduce this task to testing junta truncation. Given $(\bm X, \bm \ell)$, 
    we take the first subset of $m/100$ samples $\bm x \in \bm X$ whose label in $\bm \ell$ is 1; in
    both cases we will have at least $m/100$ such samples with high probability. 
    Now observe,
    \begin{enumerate}
        \item If $\bm \ell$ was chosen uniformly at random, then the subset of samples we send is
            sampled from the uniform distribution over $\zo^n$.
        \item If $\bm \ell = \bm{f}_{\bm S}$ then the subset of samples we send is by definition
            drawn from a $k$-junta truncation of the uniform distribution.
    \end{enumerate}
    If the junta truncation tester succeeds using $m/100$ samples, then it will succeed in
    distinguishing these cases. Choosing $(\cJ, \cF)$ as in our proof of \cref{thm:intro-main} therefore
    produces the desired lower bound.
\end{proof}

\section{Missing Proofs from \cref{section:lb-testing}}
\label{section:missing-proofs-lb-testing}

\subsection{Lower Bound on Junta Testing}

\begin{proof}[Proof of \cref{lemma:lb-k-junta}]
Write $\bm{f}_{\bm S} \colon \zo^n \to \zo$ for the distribution of the function drawn from
$\cD(\cJ,\cF)$.  Let $\bm g \colon \zo^n \to \zo$ be a uniformly random function. Then
\begin{claim}
\label{claim:lb-junta-far}
For every constant $\epsilon \in (0,1/4)$, sufficiently large $n$, and $k < n-2$, $\bm g$ is
$\epsilon$-far from being a $k$-junta with probability at least $99/100$.
\end{claim}
\begin{proof}[Proof of claim]
The number of $k$-juntas is at most $\binom{n}{k} \cdot 2^{2^k}$ and the number of functions
$\epsilon$-close to being a $k$-junta is at most
\[
\binom{n}{k} \cdot 2^{2^k} \cdot \binom{2^n}{\epsilon 2^n} \leq n^k 2^{2^k} 2^{\epsilon 2^n \log(e/\epsilon)}
\leq n^k 2^{2^n( \tfrac{1}{8}  + \tfrac{1}{4} \log(4e) )} .
\]
On the other hand, the number of functions is $2^{2^n}$. For sufficiently large $n$, the probability
that $\bm g$ is $\epsilon$-close to being a $k$-junta is at most $1/100$.
\end{proof}

\begin{claim}
For every $X = (x_1, \dotsc, x_m)$,
\[
  \| \bm{g}(X) - \unif(\zo^m) \|_\TV \leq \| \bm{f}_{\bm S}(X) - \unif(\zo^m) \|_\TV .
\]
\end{claim}
\begin{proof}[Proof of claim]
We say a label $\ell \in \zo^m$ is feasible if for all $i,j \in [m]$, $x_i = x_j \implies \ell_i =
\ell_j$, \ie any two identical sample points are assigned the same label. Let $F$ be the set of
feasible labels. Then $\bm{g}(X)$ is uniform over all feasible labels. $\bm{f}_{\bm S}(X)$ is
supported on the set of feasible labels, so writing  $p(\ell) = \Pr{ \bm{f}_{\bm S}(X) = \ell}$,
\begin{align*}
  \|\bm{f}_{\bm S}(X) - \unif(\zo^m)\|_\TV
  &= \sum_{\ell \in F} \left| p(\ell) - \frac{1}{N} \right|
  \geq \sum_{\ell \in F} \left(p(\ell) - \frac{1}{N}\right)
  = 1 - \frac{|F|}{N} \\
  &= \sum_{\ell \in F} \left( \frac{1}{|F|} - \frac{1}{N} \right)
  = \|\bm{g}(X) - \unif(\zo^m) \|_\TV . \qedhere
\end{align*}
\end{proof}

Now let $A$ be any algorithm which receives $m$ samples $\bm X$ together with a set of labels $\bm f_{\bm S}(\bm X)$. By
\cref{prop:lb-junta-setups},
\begin{align*}
  &\left| \Pru{\bm X, \bm{f}_{\bm S}}{A(\bm X, \bm{f}_{\bm S}(\bm X)) = 1}
        - \Pru{\bm X, \bm g}{ A(\bm X, \bm g(\bm X)) = 1 } \right| \\
  &\qquad\leq \Exu{\bm X}{ \left| \Pru{\bm{f}_{\bm S}}{ A(\bm X, \bm{f}_{\bm S}(\bm X)) = 1 }
        - \Pru{\bm g}{ A( \bm X, \bm{g}(\bm X)) = 1} \right| } \\
  &\qquad\leq \Exu{\bm X}{ \frac{1}{2} \|\bm{f}_{\bm S}(\bm X) - \bm{g}(\bm X)\|_\TV } \\
  &\qquad\leq \frac{1}{2} \Exu{\bm X}{ \|\bm{f}_{\bm S}(\bm X) - \unif(\zo^m) \|_\TV
        + \| \bm{g}(\bm X) - \unif(\zo^m)\|_\TV } \\
  &\qquad\leq \Exu{\bm X}{ \|\bm{f}_{\bm S}(\bm X) - \unif(\zo^m) \|_\TV }
  \leq \frac{1}{100} + \epsilon
\end{align*}
But from \cref{claim:lb-junta-far}, if $A$ was an $\epsilon$-tester then it should have $\Pru{\bm X,
\bm{f}_{\bm S}}{A(\bm X, \bm{f}_{\bm S}(\bm X)) = 1} \geq 2/3$ and $\Pru{\bm X, \bm g}{A(\bm X, \bm
g(\bm X) = 1} < 1/3$, so any algorithm using only $m$ samples cannot succeed as a tester.
\end{proof}

To apply \cref{lemma:lb-k-feature}, we require that a random $k$-junta drawn from our choice of
junta setup $(\cJ, \cF)$ is $\epsilon$-far from being a $(k/2)$-junta. This holds for our choices of
$(\cJ, \cF)$, where either $\cJ$ is the parity function or $\cJ$ is the set of all balanced
functions.
\begin{proposition}
\label{prop:lb-feature-counting}
Let $\cJ$ be a subset of functions $\zo^k \to \zo$ which is either the singleton set containing only
the parity function, or the set of all balanced functions. Then for any sufficiently small constant
$\epsilon > 0$, a uniformly random function $f \sim \cJ$ is $\epsilon$-far from being a
$\lfloor k/2 \rfloor$-junta with probability at least $99/100$.
\end{proposition}
\begin{proof}
This is trivial for the parity function. For the set of all balanced functions, a counting argument
suffices. The number of functions $\zo^k \to \zo$ which are $\epsilon$-close to being a $\lfloor k/2
\rfloor$-junta is at most
\begin{align*}
\binom{k}{k/2} 2^{2^{k/2}} \binom{2^k}{\epsilon 2^k} \leq (2e)^{k/2} (e/\epsilon)^{\epsilon 2^k}
\end{align*}
while the number of balanced functions $\zo^k \to \zo$ is at least
\[
\binom{2^k}{2^{k-1}} \geq 2^{2^{k-1}} .
\]
Therefore it suffices to take $\epsilon > 0$ to be a sufficiently small constant.
\end{proof}

\subsection{Lower Bound on Feature Selection}

\begin{proof}[Proof of \cref{lemma:lb-k-feature}]
Consider the following two distributions over $k$-juntas on $2n$ bits. Since $\cJ \subset
\binom{[n]}{k}$ we may consider two copies $\cJ_1, \cJ_2$ of $\cJ$, where $\cJ_1$ is $\cJ$ on the
first half of the $2n$ bits, and $\cJ_2$ is $\cJ$ on the second half of the $2n$ bits. This gives
two balanced $k$-junta setups $(\cJ_1, \cF)$ and $(\cJ_1, \cF)$, where all $k$-juntas in the first
setup depend only on variables in $[n]$, and all $k$-juntas in the second setup depend only on
variables in $[2n] \setminus [n]$. Then we consider distributions $\cD_1 \define \cD(\cJ_1, \cF)$
and $\cD_2 = \cD(\cJ_2, \cF)$. 

The next claim shows that a $k$-feature selector is able to distinguish between functions drawn from
$\cD_1$ and those drawn from $\cD_2$.

\begin{claim}
Suppose $A$ is a $k$-feature selector for domain $\zo^{2n}$ with sample complexity $m$ on parameters
$n,k,\epsilon$.  Then for $f_1 \sim \cD_1$, with probability at least $2/3$ over $f_1$ and the
samples $\bm X$, $A$ outputs a set $S \in \binom{[2n]}{k}$ such that $|S \cap [n]| > k/2$.
Similarly, for $f_2 \sim \cD_2$, with probability at least $2/3$ over $f_2$ and the samples $\bm X$,
$A$ outputs a set $S$ with $|S \cap [n]| < k/2$.
\end{claim}
\begin{proof}
For each $f_1$ in the support of $\cD_1$, $A$ has probability at least $3/4$ over $\bm X$ of
choosing a set $S$ such that $f_1$ is $\epsilon$-close to a $k$-junta on variables $S$. If $f_1$ is
$\epsilon$-far from every $(k/2)$-junta, then $A$ has probability at least $3/4$ over $\bm X$
of choosing a set $S$ such that $S \cap [n] > k/2$. By assumption, the probability that $f_1 \sim
\cD_1$ is $\epsilon$-far from being a $(k/2)$-junta is at least $9/10$, so the claim follows.
\end{proof}

From the claim, we may conclude that the TV distance between the distribution of $\bm f_1(\bm X)$
and the distribution of $\bm f_2(\bm X)$ is at least $1/3$, since the probability of the event $|\bm
S \cap [n]| < k/2$ differs by $1/3$ between these two distributions. But this contradicts
\cref{prop:lb-junta-setups}, which implies that the TV distance between these two distributions is
at most $2(\tfrac{1}{100} + \epsilon)$.
\end{proof}

\subsection{Upper Bound on $\epsilon$ Dependence for the Uniform Distribution}
\label{section:lb-eps-unif}

We now show how to get better dependence on $\epsilon$ in the case where the distribution
is known to be the uniform distribution.

\begin{claim}
\label{claim:unif-junta-smallest}
For any two distinct $k$-juntas $f, g \colon \zo^n \to \zo$,
\[
  \dist_\unif(f,g) \geq 2^{-k} .
\]
\end{claim}
\begin{proof}[Proof of claim]
Let $f$ be a $k$-junta on variables $S \in \binom{n}{k}$ and $g$ be a $k$-junta on variables $T \in
\binom{n}{k}$. Write $\Delta \define |S \cap T|$. Since $f, g$ are distinct there exist $z \in
\zo^\Delta$ and $x \in \zo^n$ such that $x_{S \cap T} = z$ and $f(x) \neq g(x)$. Without loss of
generality we may assume $f(x) = 1, g(x) = 0$. Write
\begin{align*}
\alpha &\define \Pruc{\bm y}{ f(\bm y) = 1 }{ \bm{y}_{S \cap T} = z } ,\\
\beta &\define \Pruc{\bm y}{ g(\bm y) = 1 }{ \bm{y}_{S \cap T} = z } .
\end{align*}
Note that $\alpha \geq 2^{\Delta-k}$ and $\beta \leq 1 - 2^{\Delta-k}$. The expression
\begin{equation}
\label{eq:ub-unif-ab}
  \alpha(1-\beta) + (1-\alpha)\beta = \alpha + \beta - 2\alpha\beta 
\end{equation}
is minimized when either both $\alpha, \beta$ attain their minimum values, or both attain their
maximum values. In each case the lower bound on \eqref{eq:ub-unif-ab} is $2^{\Delta-k}$. Then
\begin{align*}
\Pru{\bm y}{f(\bm y) \neq g(\bm y)}
&\geq \Pr{\bm{y}_{S \cap T} = z} \Pruc{}{f(\bm y) \neq g(\bm y)}{\bm{y}_{S \cap T} = z} \\
&\geq 2^{-\Delta} (\alpha(1-\beta) + (1-\alpha)\beta) \geq 2^{-k} . \qedhere
\end{align*}
\end{proof}

\paragraph*{Upper bound on feature selection.}
This follows from the same reduction to SOPP as in
\cref{res:ub-reduction-to-sopp}, except that if $\epsilon < 2^{-k}$ we do the reduction with
parameter $\epsilon^* = 2^{-k}$ instead of $\epsilon$. In this case, since $\dist_\unif(f,g) \geq
2^{-k}$ for any two distinct $k$-juntas $f$ and $g$ (\cref{claim:unif-junta-smallest}), this algorithm will output (with probability at
least $2/3$) the exact set of relevant variables. Via \cref{res:ub-sopp}, this gives an upper bound
of
\[
  O\left( \min\left\{1/\epsilon, 2^k\right\} \cdot \left( \sqrt{2^k \log\binom{n}{k}} + \log\binom{n}{k}
\right) \right).
\]

\paragraph*{Upper bound on testing juntas.}
The tester is as follows. On inputs $f
\colon \zo^n \to \zo$ and $\epsilon > 0$:
\begin{enumerate}
\item If $\epsilon \geq 2^{-k}$, use exactly the same algorithm as in
\cref{res:ub-reduction-to-sopp}. 
\item Otherwise, if $\epsilon < 2^{-k}$, use the above algorithm for $k$-feature selection, which
returns (with probability at least $3/4$) a set $S \in \binom{n}{k}$ with the property that, if $f$
is a $k$-junta, then $S$ is the exact set of relevant variables of $f$, due to
\cref{claim:unif-junta-smallest}. Now, use
$O( \sqrt{2^k}/\epsilon)$ samples to run the SOPP tester with domain size $N = 2^k$ and error
parameter $\epsilon$ on the distribution obtained from variables $S$ as in the reduction in
\cref{res:ub-reduction-to-sopp}.
\end{enumerate}
Together, these give the upper bound of
\[
  O\left( \min\left\{ \frac{1}{\epsilon}, 2^k \right\}
    \cdot  \left(\sqrt{2^k\log\binom{n}{k}} + \log\binom{n}{k}\right) + \frac{1}{\epsilon}\sqrt{2^k}\right) . 
\]